%% file: main.tex
\documentclass{article}

 \usepackage[final]{neurips_2021}

\usepackage[utf8]{inputenc} 
\usepackage[T1]{fontenc}    
\usepackage{hyperref}       
\usepackage{url}            
\usepackage{booktabs}       
\usepackage{amsfonts}       
\usepackage{nicefrac}       
\usepackage{microtype}      
\usepackage{adjustbox}
\usepackage[ruled]{algorithm2e}
\usepackage{algorithmic}

\usepackage{microtype}
\usepackage{graphicx}
\usepackage{booktabs} 

\usepackage{subcaption}

\usepackage{hyperref}

\usepackage{amssymb}
\usepackage{amsfonts}
\usepackage{xifthen}
\usepackage{xparse}
\usepackage{dsfont}
\usepackage{xspace}
\usepackage{amsmath}
\usepackage{xcolor}
\usepackage{wrapfig}
\usepackage{enumitem}

\usepackage{amsthm}
\newtheorem{theorem}{Theorem}
\newtheorem{lemma}[theorem]{Lemma}
\newtheorem{corollary}[theorem]{Corollary}

\newenvironment{proofof}[1]{\begin{trivlist} \item {\bf Proof
#1:~~}}
  {\qed\end{trivlist}}

\newcommand{\lr}[1]{\left (#1\right)}
\newcommand{\lrs}[1]{\left [#1 \right]}
\newcommand{\lrc}[1]{\left \{#1 \right\}}

\newcommand{\V}{\mathbb{V}}
\newcommand{\HH}{\mathcal{H}}
\NewDocumentCommand{\1}{o}{\mathds 1{\IfValueT{#1}{\lr{#1}}}}
\NewDocumentCommand{\E}{o}{\mathbb E\IfValueT{#1}{\lrs{#1}}}
\NewDocumentCommand{\Var}{o}{\V\IfValueT{#1}{\lrs{#1}}}
\let\P\undefined
\NewDocumentCommand{\P}{o}{\mathbb P{\IfValueT{#1}{\lr{#1}}}}

\newcommand{\Vmin}{V_{\min}}

\newcommand{\dataset}[1]{#1}

\DeclareMathOperator{\MV}{MV}
\DeclareMathOperator{\KL}{KL}
\DeclareMathOperator{\kl}{kl}
\newcommand{\D}{\mathbb{D}}
\newcommand{\R}{\mathbb{R}}

\newcommand{\FO}{\operatorname{FO}}
\newcommand{\TND}{\operatorname{TND}}

\newcommand{\CMUTND}{\operatorname{CCTND}}
\newcommand{\COTND}{\operatorname{CCPBB}}
\newcommand{\CCTND}{\operatorname{CCTND}}
\newcommand{\CCPBB}{\operatorname{CCPBB}}
\newcommand{\BOUND}{\operatorname{BOUND}}

\DeclareMathOperator{\testset}{S_{test}}

\title{Chebyshev-Cantelli PAC-Bayes-Bennett Inequality for the Weighted Majority Vote}

\author{%
   Yi-Shan Wu\\
   University of Copenhagen\\
   \texttt{yswu@di.ku.dk}\\
   \And
    Andr{\'e}s R. Masegosa\\
    University of Aalborg \\
   \texttt{arma@cs.aau.dk} \\
   \And
   Stephan S. Lorenzen \\
   University of Copenhagen \\
   \texttt{lorenzen@di.ku.dk} \\
   \And
   Christian Igel \\
   University of Copenhagen \\
   \texttt{igel@di.ku.dk} \\
   \And
   Yevgeny Seldin \\
   University of Copenhagen \\
   \texttt{seldin@di.ku.dk} \\
}

\begin{document}

\maketitle

\begin{abstract}
We present a new second-order oracle bound for the expected risk of a weighted majority vote. The bound is based on a novel parametric form of the Chebyshev-Cantelli inequality (a.k.a.\ one-sided Chebyshev's), which is amenable to efficient minimization. The new form resolves the optimization challenge faced by prior oracle bounds based on the Chebyshev-Cantelli inequality, the C-bounds \citep{GLL+15}, and, at the same time, it improves on the oracle bound based on second order Markov's inequality introduced by \citet{MLIS20}. We also derive a new concentration of measure inequality, which we name PAC-Bayes-Bennett, since it combines PAC-Bayesian bounding with Bennett's inequality. We use it for empirical estimation of the oracle bound. The PAC-Bayes-Bennett inequality improves on the PAC-Bayes-Bernstein inequality of \citet{SLCB+12}.  We provide an empirical evaluation demonstrating that the new bounds can improve on the work of \citet{MLIS20}. Both the parametric form of the Chebyshev-Cantelli inequality and the PAC-Bayes-Bennett inequality may be of independent interest for the study of concentration of measure in other domains.

\end{abstract}

\section{Introduction}

Weighted majority vote is a central technique for combining predictions of multiple classifiers. It is an integral part of random forests \citep{Bre96,Bre01}, boosting \citep{FS96}, gradient boosting \citep{Fri99,friedman2001greedy,MBBF99,chen2016xgboost}, and it is also used to combine predictions of heterogeneous classifiers. It is part of the winning strategies in many machine learning competitions. The power of the majority vote is in the cancellation of errors effect: when the errors of individual classifiers are independent or anticorrelated and the error probability of individual classifiers is smaller than $0.5$, then the errors average out and the majority vote tends to outperform the individual classifiers.

Generalization bounds for weighted majority vote and theoretically-grounded approaches to weight-tuning are decades-old research topics. \citet{BK16} derived an optimal solution under the assumption of known error rates and independence of errors of individual classifiers, but neither of the two assumptions is typically satisfied in practice.

In absence of the independence assumption, the most basic result is the first order oracle bound, which is based on Markov's inequality and bounds the expected loss of $\rho$-weighted majority vote by twice the $\rho$-weighted average of expected losses of the individual classifiers. This finding is so old and basic that \citet{LST02} call it ``the folk theorem''. The $\rho$-weighted average of the expected losses is then bounded using PAC-Bayesian bounds, turning the oracle bound into an empirical bound \citep{McA98,See02,LST02}. While the translation from oracle to empirical bound is quite tight \citep{GLLM09,TIWS17}, the first order oracle bound ignores the correlation of errors, which is the main power of the majority vote. As a result, its minimization overconcentrates the weights on the best-performing classifiers, effectively reducing the majority vote to a very few or even just a single best classifier, which leads to a significant deterioration of the test error on held-out data \citep{LIS19,MLIS20}.

In order to take correlation of errors into account, \citet{LLM+07} derived second order oracle bounds, the C-bounds, which are based on the Chebyshev-Cantelli inequality. The ideas were further developed by \citet{LMR11}, \citet{GLL+15}, and \citet{LMRR17}. However, they were only able to optimize the bounds in the highly restrictive setting of binary classification with self-complemented sets of classifiers and aligned priors and posteriors \citep{GLL+15}. Several follow-up works resorted to minimization of heuristic surrogates rather than the bound itself \citep{BCRL20,VGHM21}. Furthermore, second order oracle quantities in the denominator of the oracle bounds lead to looseness in their translation to empirical bounds \citep{LIS19}.

\citet{MLIS20} proposed an alternative second-order oracle bound, the tandem bound, based on second order Markov's inequality. While the second order Markov's inequality is weaker than the Chebyshev-Cantelli inequality, the resulting bound has no oracle quantities in the denominator, which allows tight translation to an empirical bound. Additionally, \citeauthor{MLIS20} proposed an efficient procedure for minimization of their empirical bound. They have shown that minimization of the empirical bound does not lead to deterioration of the test error.

Our work can be seen as a bridge between the tandem bound and the C-bounds, and as an improvement of both. The key novelty is a new parametric form of Chebyshev-Cantelli inequality, which preserves the tightness of Chebyshev-Cantelli, but avoids oracle quantities in the denominator. This allows both efficient translation to empirical bounds and efficient minimization. We derive two new second order oracle bounds based on the new inequality, one using the tandem loss and the other using the tandem loss with an offset. For empirical estimation of the latter we derive a PAC-Bayes-Bennett inequality. The overall contributions can be summarized as follows:
\begin{enumerate}[leftmargin=*]
    \item We propose a new parametric form of the Chebyshev-Cantelli inequality, which has no variance in the denominator and preserves tightness of the original bound. The new form allows efficient minimization and empirical estimation.
    \item We propose two new second order oracle bounds for the weighted majority vote based on the new form of the Chebyshev-Cantelli inequality. The bounds have two advantages: (1) they are amenable to tight translation to empirical bounds; and (2) the resulting empirical bounds are amenable to efficient minimization.
    \item We derive a new concentration of measure inequality, which we name the PAC-Bayes-Bennett inequality. It improves on the PAC-Bayes-Bernstein inequality of \citet{SLCB+12}. We use the inequality for bounding the tandem loss with an offset.
\end{enumerate}

\section{Problem setup}
\label{sec:generalsetup}

The problem setup and notations are borrowed from \citet{MLIS20}. 

\paragraph{Multiclass classification.}
We let $S=\{(X_1,Y_1),\ldots,(X_n,Y_n)\}$ be an independent identically
distributed sample from $\cal{X}\times\cal{Y}$, drawn according to an unknown distribution $D$, where $\cal{Y}$ is finite and $\cal{X}$ is arbitrary. A hypothesis
is a function $h : \cal{X} \rightarrow \cal{Y}$, and $\cal H$ denotes a space of hypotheses. We evaluate the quality of $h$  using the zero-one loss $\ell(h(X),Y)=\1[h(X)\neq Y]$, where $\1[\cdot]$ is the indicator function. The expected loss of $h$ is denoted by $L(h) =
\mathbb{E}_{(X,Y)\sim D} [\ell(h(X),Y)]$ and the empirical loss of $h$ on a sample $S$ of size $n$ is denoted by $\hat{L}(h,S) =
\frac{1}{n} \sum_{i=1}^n \ell (h(X_i),Y_i)$. 

\paragraph{Randomized classifiers.}
A \emph{randomized classifier} (a.k.a.\ Gibbs classifier) associated with a distribution $\rho$ on $\cal{H}$, for each input $X$ randomly draws a hypothesis $h\in{\cal H}$ according to $\rho$ and predicts $h(X)$. The expected loss of a randomized classifier is given by $\mathbb{E}_{h\sim \rho} [L(h)]$ and the empirical loss by $\mathbb{E}_{h\sim\rho}[\hat{L}(h,S)]$. 
To simplify the notation we use $\E_D[\cdot]$ as a shorthand for $\E_{(X,Y)\sim D}[\cdot]$ and $\E_\rho[\cdot]$ as a shorthand for $\E_{h\sim \rho}[\cdot]$. 

\paragraph{Ensemble classifiers and majority vote.}
Ensemble classifiers predict by taking a weighted aggregation of predictions by hypotheses from ${\cal H}$. The $\rho$-weighted majority vote $\MV_\rho$ predicts $\MV_\rho (X)= \arg\max_{y\in{\cal Y}} \E_\rho[\1[h(X) = y]]
$, where ties can be resolved arbitrarily. 

\section{A review of prior first and second order oracle bounds}

If majority voting makes an error, we know that at least a $\rho$-weighted half of the classifiers have made an error and, therefore, $\ell(\MV_\rho(X),Y) \leq \1[\E_\rho[\1[h(X)\neq Y]] \geq 0.5]$. This observation leads to the well-known first order oracle bound for the loss of weighted majority vote.
\begin{theorem}[First Order Oracle Bound]
\label{thm:first-order}
\[
L(\MV_\rho)\leq 2\E_\rho[L(h)].
\]
\end{theorem}
\begin{proof}
We have $L(\MV_\rho) = \E_D[\ell(\MV_\rho(X),Y)] \leq \P[\E_\rho[\1[h(X)\neq Y]] \geq 0.5]$.  By applying Markov's inequality to random variable $Z = \E_\rho[\1[h(X)\neq Y]]$ we have: 
\[
L(\MV_\rho) \leq \P[\E_\rho[\1[h(X)\neq Y]] \geq 0.5]
\leq 2\E_D[\E_\rho[\1[h(X)\neq Y]]]
= 2\E_\rho[L(h)].
\qedhere
\]
\end{proof}
PAC-Bayesian analysis can be used to bound $\E_\rho[L(h)]$ in Theorem~\ref{thm:first-order} in terms of $\E_\rho[\hat L(h,S)]$, thus turning the oracle bound into an empirical one. The disadvantage of the first order approach is that $\E_\rho[L(h)]$ ignores correlations of predictions, which is the main power of the majority vote.

\citet{MLIS20} have used second order Markov's inequality, by which for a non-negative random variable $Z$ and $\varepsilon > 0$
\[
\P[Z\geq \varepsilon] = \P[Z^2 \geq \varepsilon^2] \leq \frac{\E[Z^2]}{\varepsilon^2}.
\]
For pairs of hypotheses $h$ and $h'$ they have defined the \emph{tandem loss} $\ell(h(X),h'(X),Y) = \1[h(X)\neq Y \wedge h'(X) \neq Y]=\1[h(X)\neq Y]\1[h'(X) \neq Y]$, also termed \emph{joint error} by \citet{LLM+07}, which counts an error only if both $h$ and $h'$ err on a sample $(X,Y)$. The corresponding \emph{expected tandem loss} is defined by
\[
L(h,h') = \E_D[\1[h(X)\neq Y]\1[h'(X)\neq Y]].
\]
\citet{LLM+07} and \citet{MLIS20} have shown that expectation of the second moment of the weighted loss equals expectation of the tandem loss. Using $\rho^2$ as a shorthand for the product distribution $\rho \times \rho$ over $\HH \times \HH$ and $\E_{\rho^2}[L(h,h')]$ as a shorthand for $\E_{h\sim\rho,h'\sim\rho}[L(h,h')]$, the result is the following.
\begin{lemma}[\citealp{MLIS20}]
\label{lem:tandemloss}
In multiclass classification
\[
\E_D[\E_\rho[\1[h(X) \neq Y]]^2] = \E_{\rho^2}[L(h,h')].
\]
\end{lemma}
By combining second order Markov's inequality with Lemma~\ref{lem:tandemloss}, \citeauthor{MLIS20} have shown the following result.
\begin{theorem}[\citealp{MLIS20}]
\label{thm:tandembound}
In multiclass classification
\[
L(\MV_\rho) \leq 4\E_{\rho^2}[L(h,h')].
\]
\end{theorem}

\citet{LLM+07} have used the Chebyshev-Cantelli inequality to derive a different form of a second order oracle bound. We use $\Var[Z]$ to denote the variance of a random variable $Z$ in the statement of Chebyshev-Cantelli inequality.
\begin{theorem}[Chebyshev-Cantelli inequality]
\label{thm:CC}
For $\varepsilon > 0$
\[
\P[Z - \E[Z] \geq \varepsilon] \leq \frac{\Var[Z]}{\varepsilon^2+\Var[Z]}.
\]
\end{theorem}
Theorem~\ref{thm:CC} together with Lemma~\ref{lem:tandemloss} leads to the following result, known as the \emph{C-bound}.
\begin{theorem}[\citealp{LLM+07,MLIS20}]
\label{thm:Cbound}
If $\E_\rho[L(h)] \leq \frac{1}{2}$, then
\[
L(\MV_\rho)\leq \frac{\E_{\rho^2}[L(h,h')] - \E_\rho[L(h)]^2}{\frac{1}{4} + \E_{\rho^2}[L(h,h')] - \E_\rho[L(h)]}.
\]
\end{theorem}
\citeauthor{MLIS20} have shown that the  Chebyshev-Cantelli inequality is always at least as tight as  second order Markov's inequality (below we provide an alternative and more intuitive proof of this fact) and, therefore, the oracle bound in Theorem~\ref{thm:Cbound} is always at least as tight as the oracle bound in Theorem~\ref{thm:tandembound}. However, the presence of $\E_{\rho^2}[L(h,h')]$ and $\E_\rho[L(h)]$ in the denominator make empirical estimation and optimization of the bound in Theorem~\ref{thm:Cbound} impractical, and Theorem~\ref{thm:tandembound} was the only practically applicable second order bound so far.

\section{Main Contributions}

We present three main contributions: (1) a new form of the Chebyshev-Cantelli inequality, which is convenient for optimization; (2) an application of the inequality to the analysis of weighted majority vote; and (3) a PAC-Bayes-Bennett inequality, which is used to bound the risk with an offset in the bound for weighted majority vote. We start with the new form of Chebyshev-Cantelli inequality, which can be seen as a refinement of second order Markov's inequality or as an intermediate step in the proof of the Chebyshev-Cantelli inequality.
\begin{theorem}
\label{thm:CCmu}
For any $\varepsilon > 0$ and all $\mu < \varepsilon$
\[
\P[Z \geq \varepsilon] \leq \frac{\E[(Z-\mu)^2]}{(\varepsilon - \mu)^2}.
\]
\end{theorem}
\begin{proof}
\[
\P[Z\geq \varepsilon] = \P[Z - \mu \geq \varepsilon - \mu] \leq \P[(Z-\mu)^2 \geq (\varepsilon - \mu)^2]
\leq \frac{\E[(Z-\mu)^2]}{(\varepsilon - \mu)^2}.
\qedhere
\]
\end{proof}
The inequality can also be written as
\begin{equation}
\label{eq:CC2}
\P[Z \geq \varepsilon] \leq \frac{\E[(Z-\mu)^2]}{(\varepsilon - \mu)^2} = \frac{\E[Z^2] - 2\mu\E[Z]+\mu^2}{(\varepsilon - \mu)^2}.
\end{equation}
The bound is minimized by $\mu^* = \E[Z]-\frac{\Var[Z]}{\varepsilon-\E[Z]}$, which can be verified by taking a derivative of the bound with respect to $\mu$. Note that $\mu^*$ can take negative values. Substitution of $\mu^*$ into the bound and simple algebraic manipulations recover the Chebyshev-Cantelli inequality, whereas $\mu=0$ recovers  second order Markov's inequality. The main advantage of Theorem~\ref{thm:CCmu} over the  Chebyshev-Cantelli inequality is ease of estimation and optimization due to absence of the variance term in the denominator.

Equation \eqref{eq:CC2} leads to two new  second order oracle bounds for the weighted majority vote, given in Theorems~\ref{thm:MV} and \ref{thm:MVBen}.
\begin{theorem}
\label{thm:MV}
In multiclass classification, for all $\rho$ and all $\mu<0.5$
\[
L(\MV_\rho)\leq \frac{\E_{\rho^2}[L(h,h')] - 2\mu\E_\rho[L(h)]+\mu^2}{(0.5 - \mu)^2}.
\]
\end{theorem}
\begin{proof}
As in the previous section, we take $Z = \E_\rho[\1[h(X)\neq Y]]$, so that $L(\MV_\rho)\leq \P(Z\geq 0.5)$. The result follows by \eqref{eq:CC2} and the calculations of $\E[Z^2]$ and $\E[Z]$ from the previous section. Note that the result is a deterministic statement.
\end{proof}
For $\mu=0$, Theorem~\ref{thm:MV} recovers Theorem~\ref{thm:tandembound}, but if $\mu^*=\E_\rho[L(h)]-\frac{\E_{\rho^2}[L(h,h')]-\E[L(h)]^2}{0.5 - \E_\rho[L(h)]} \neq 0$, then substitution of $\mu^*$ into the theorem yields a tighter oracle bound. At the same time, substitution of $\mu^*$ recovers Theorem~\ref{thm:Cbound}, but the great advantage of Theorem~\ref{thm:MV} is that the bound allows easy empirical estimation and optimization with respect to $\rho$, due to absence of $\E_{\rho^2}[L(h,h')]$ and $\E_\rho[L(h)]$ in the denominator. Thus, Theorem~\ref{thm:MV} has the oracle tightness of Theorem~\ref{thm:Cbound} and the ease of estimation and optimization of Theorem~\ref{thm:tandembound}. The oracle quantities $\E_{\rho^2}[L(h,h')]$ and $\E_\rho[L(h)]$ can be bounded using PAC-Bayes-kl or PAC-Bayes-$\lambda$ inequalities, as discussed in the next section.

In order to present the second oracle bound we introduce a new quantity. For a pair of hypotheses $h$ and $h'$ and a constant $\mu$, we define \emph{tandem loss with $\mu$-offset}, for brevity \emph{$\mu$-tandem loss}, as
\begin{equation}
\label{eq:ell-mu}    
    \ell_\mu(h(X),h'(X),Y) =  (\1[h(X)\neq Y] - \mu)(\1[h'(X)\neq Y] - \mu).
\end{equation}
Note that it can take negative values. We denote its expectation by 
\[
L_{\mu}(h,h') = \E_D[\ell_\mu(h(X),h'(X),Y)] = \E_D[(\1[h(X)\neq Y] - \mu)(\1[h'(X)\neq Y] - \mu)].
\] 
With $Z = \E_\rho[\1[h(X)\neq Y]]$ as before, we have
\begin{align*}
    \E[(Z-\mu)^2] &= \E_D[(\E_\rho[(\1[h(X)\neq Y] - \mu)])^2]
    \\
    &= \E_{\rho^2}[\E_D[(\1[h(X)\neq Y] - \mu)(\1[h'(X)\neq Y] - \mu)]] 
    = \E_{\rho^2}[L_{\mu}(h,h')].
\end{align*}

Now we present our second oracle bound.
\begin{theorem}
\label{thm:MVBen}
In multiclass classification, for all $\rho$ and all $\mu<0.5$
\[
L(\MV_\rho)\leq \frac{\E_{\rho^2}[L_\mu(h,h')]}{(0.5 - \mu)^2}.
\]
\end{theorem}
\begin{proof}
The result follows by Theorem~\ref{thm:CCmu} and the calculation above. Note that the inequality is a deterministic statement.
\end{proof}
In order to discuss the advantage of Theorem~\ref{thm:MVBen}, we define the variance of the $\mu$-tandem loss 
\[
\Var_\mu(h,h') = \E_D[\lr{(\1[h(X)\neq Y] - \mu)(\1[h'(X)\neq Y] - \mu) - L_\mu(h,h')}^2].
\]
If the variance of the $\mu$-tandem loss is small, we can use Bernstein-type inequalities to obtain tighter estimates compared to kl-type inequalities. 

We bound the $\mu$-tandem loss using our next contribution, the PAC-Bayes-Bennett inequality, which improves on the PAC-Bayes-Bernstein inequality derived by \citet{SLCB+12} and may be of independent interest. The inequality holds for any loss function with bounded length of the range, we use $\tilde \ell$ and matching tilde-marked quantities to distinguish it from the zero-one loss $\ell$. We let $\tilde L(h)=\E_D[\tilde \ell(h(X),Y)]$ and $\tilde \Var(h) = \E_D[(\tilde \ell(h(X),Y) - \tilde L(h))^2]$ be the expected tilde-loss of $h$ and its variance and let $\hat {\tilde L}(h,S) = \frac{1}{n}\sum_{i=1}^n \tilde \ell(h(X_i),Y_i)$ be the empirical tilde-loss of $h$ on a sample $S$.
\begin{theorem}[PAC-Bayes-Bennett inequality]
\label{thm:PBBennett}
Let $\tilde \ell(\cdot,\cdot)$ be an arbitrary loss function taking values in an interval of length $b$, and assume that $\tilde \V(h)$ is finite for all $h$. Let $\phi(x) = e^x - x - 1$. Then for any distribution $\pi$ on $\HH$ that is independent of $S$ and any $\gamma > 0$ and $\delta\in(0,1)$, with probability at least $1-\delta$ over a random draw of $S$, for all distributions $\rho$ on $\HH$ simultaneously:
\[
\E_\rho[\tilde L(h)] \leq \E_\rho[\hat{\tilde L}(h,S)] + \frac{\phi(\gamma b)}{\gamma b^2}\E_\rho[\tilde \Var(h)] + \frac{\KL(\rho\|\pi)+\ln \frac{1}{\delta}}{\gamma n}.
\]
\end{theorem}
The proof is based on a change of measure argument combined with Bennett's inequality, the details are provided in Appendix~\ref{app:Bennett}. Note that the result holds for a fixed (but arbitrary) $\gamma>0$. In case of optimization with respect to $\gamma$ a union bound has to be applied. For a fixed $\rho$ the bound is convex in $\gamma$ and for a fixed $\gamma$ it is convex in $\rho$, although it is not necessarily jointly convex in $\rho$ and $\gamma$. See Appendix~\ref{app:minimization_details} for optimization details. The PAC-Bayes-Bennett inequality is identical to the PAC-Bayes-Bernstein inequality of \citet[Theorem 7]{SLCB+12}, except that in the latter the coefficient in front of $\E_\rho[\tilde \Var[h]]$ is $(e-2)\gamma$ instead of $\frac{\phi(\gamma b)}{\gamma b^2}$. The result improves on the result of \citeauthor{SLCB+12} in two ways. First, in the result of \citeauthor{SLCB+12} $\gamma$ is restricted to the $(0, 1/b]$ interval, whereas in our result $\gamma$ is unrestricted from above. And second, we can rewrite the coefficient in front of the variance as $\frac{\phi(\gamma b)}{\gamma b^2}=\frac{\phi(\gamma b)}{\gamma^2 b^2}\gamma$, where $\frac{\phi(\gamma b)}{\gamma^2 b^2}$ is a monotonically increasing function of $\gamma$, which in the interval $\gamma\in(0,1/b]$ satisfies $\lim_{\gamma\to 0}\frac{\phi(\gamma b)}{\gamma^2 b^2} = \frac{1}{2}$ and for $\gamma=1/b$ it gives $\frac{\phi(\gamma b)}{\gamma^2 b^2} = (e-2)$. Thus, PAC-Bayes-Bennett is always at least as tight as PAC-Bayes-Bernstein and, at the same time, for $\gamma < 1/b$ it improves the constant coefficient in front of the variance from $(e-2)\approx 0.72$ down to 0.5 for $\gamma\to0$. For $\gamma > 1/b$ PAC-Bayes-Bennett also improves on PAC-Bayes-Bernstein, because PAC-Bayes-Bernstein uses the suboptimal value $\gamma=1/b$ dictated by its restricted range of $\gamma$.

\section{From oracle to empirical bounds}

We obtain empirical bounds on the oracle quantities $\E_{\rho^2}[L(h,h')]$ and $\E_\rho[L(h)]$ in Theorem~\ref{thm:MV} and $\E_{\rho^2}[L_\mu(h,h')]$ in Theorem~\ref{thm:MVBen} by using PAC-Bayesian inequalities. The empirical counterpart of the expected tandem loss is the empirical tandem loss
\[
\hat L(h,h',S) = \frac{1}{n}\sum_{i=1}^n \1[h(X_i)\neq Y_i]\1[h'(X_i) \neq Y_i].
\]
For bounding $\E_{\rho^2}[L(h,h')]$ and $\E_\rho[L(h)]$ we use either PAC-Bayes-kl or PAC-Bayes-$\lambda$ inequalities, both cited below. We use $\KL(\rho\|\pi)$ to denote the Kullback-Leibler divergence between distributions $\rho$ and $\pi$ on $\HH$ and $\kl(p\|q)$ to denote the Kullback-Leibler divergence between two Bernoulli distributions with biases $p$ and $q$.

\begin{theorem}[PAC-Bayes-kl Inequality, \citealp{See02}, \citealp{Mau04}] 
For any probability distribution $\pi$ on ${\cal H}$ that is independent of $S$ and any $\delta \in (0,1)$, with probability at least $1-\delta$ over a random draw of a sample $S$, for all distributions $\rho$ on ${\cal H}$ simultaneously:
\begin{equation}
\label{eq:PBkl}
\kl\lr{\E_\rho[\hat L(h,S)]\middle\|\E_\rho\lrs{L(h)}} \leq \frac{\KL(\rho\|\pi) + \ln(2 \sqrt n/\delta)}{n}.
\end{equation}
\label{thm:PBkl}
\end{theorem}

\begin{theorem}[PAC-Bayes-$\lambda$ Inequality, \citealp{TIWS17,MLIS20}]\label{thm:lambdabound} For any probability distribution $\pi$ on ${\cal H}$ that is independent of $S$ and any $\delta \in (0,1)$, with probability at least $1-\delta$ over a random draw of a sample $S$, for all distributions $\rho$ on ${\cal H}$ and all $\lambda \in (0,2)$ and $\gamma > 0$ simultaneously:
\begin{align}
\E_\rho\lrs{L(h)} &\leq \frac{\E_\rho[\hat L(h,S)]}{1 - \frac{\lambda}{2}} + \frac{\KL(\rho\|\pi) + \ln(2 \sqrt n/\delta)}{\lambda\lr{1-\frac{\lambda}{2}}n},\label{eq:PBlambda}\\
\E_\rho\lrs{L(h)} &\geq \lr{1 - \frac{\gamma}{2}}\E_\rho[\hat L(h,S)] - \frac{\KL(\rho\|\pi) + \ln(2 \sqrt n/\delta)}{\gamma n}.\label{eq:PBlambda-lower}
\end{align}
\label{thm:PBlambda}
\end{theorem}

(The upper bound \eqref{eq:PBlambda} is due to \citet{TIWS17} and the lower bound \eqref{eq:PBlambda-lower} is due to \citet{MLIS20}, and the two bounds hold simultaneously.) The PAC-Bayes-$\lambda$ inequality is an optimization-friendly relaxation of the PAC-Bayes-kl inequality. Therefore, for optimization of $\rho$ we use the PAC-Bayes-$\lambda$ inequality, the upper bound for $\E_{\rho^2}[L(h,h')]$ and the lower or upper bound for $\E_\rho[L(h)]$, depending on the positiveness of $\mu$, but once we have converged to a solution we use PAC-Bayes-kl to compute the final bound. The kl form provides both an upper and a lower bound through the upper and lower inverse of the kl.\footnote{\citet{RDGR18} and \citet{LGGL19} provide alternative ways of direct minimization of the upper bound on $\E_\rho[L(h)]$ given by the upper inverse of $\kl$ in the PAC-Bayes-kl bound. We use the PAC-Bayes-$\lambda$ relaxation due to its simplicity, and because it provides an easy way of simultaneous optimization of an upper bound on $\E_{\rho^2}[L(h,h')]$ and a lower or upper bound on $\E_\rho[L(h)]$ (depending on $\mu$).} Taking the oracle bound from Theorem~\ref{thm:MV} and bounding the oracle quantities using Theorem~\ref{thm:PBlambda} we obtain the following result.
\begin{theorem}\label{thm:first-mu-bound}
For any distribution $\pi$ on $\HH$ that is independent of $S$, and any $\delta\in(0,1)$, with probability at least $1-\delta$ over a random draw of $S$, for all distributions $\rho$ on $\HH$, and all $\mu, \lambda$, and $\gamma$ in the ranges specified below simultaneously, we have:
\begin{itemize}[leftmargin=*]
    \item For $\mu \in [0,0.5)$, $\lambda\in(0,2)$, and $\gamma > 0$:
\begin{multline*}
L(\MV_\rho)\leq \frac{1}{(0.5-\mu)^2}\bigg[\frac{\E_{\rho^2}[\hat L(h,h',S)]}{1 - \frac{\lambda}{2}} + \frac{2\KL(\rho\|\pi) + \ln(4 \sqrt n/\delta)}{\lambda\lr{1-\frac{\lambda}{2}}n}\\ 
- 2\mu\lr{\lr{1 - \frac{\gamma}{2}}\E_\rho[\hat L(h,S)] - \frac{\KL(\rho\|\pi) + \ln(4 \sqrt n/\delta)}{\gamma n}} + \mu^2\bigg].
\end{multline*}
\item For $\mu < 0$, $\lambda\in(0,2)$, and $\gamma \in (0,2)$:
\begin{multline*}
L(\MV_\rho)\leq \frac{1}{(0.5-\mu)^2}\bigg[\frac{\E_{\rho^2}[\hat L(h,h',S)]}{1 - \frac{\lambda}{2}} + \frac{2\KL(\rho\|\pi) + \ln(4 \sqrt n/\delta)}{\lambda\lr{1-\frac{\lambda}{2}}n}\\ 
- 2\mu\lr{ \frac{\E_{\rho}[\hat L(h,S)]}{1 - \frac{\gamma}{2}} + \frac{\KL(\rho\|\pi) + \ln(4\sqrt n/\delta)}{\gamma\lr{1-\frac{\gamma}{2}}n}} + \mu^2\bigg].
\end{multline*}
\end{itemize}
\end{theorem}
\begin{proof}
The result follows by substitution of the upper bound \eqref{eq:PBlambda} on $\E_{\rho^2}[L(h,h')]$ and the lower bound \eqref{eq:PBlambda-lower} on $\E_\rho[L(h)]$ in the case of positive $\mu$, or the upper bound \eqref{eq:PBlambda} on $\E_\rho[L(h)]$ in the case of negative $\mu$, into Theorem~\ref{thm:MV}. We note that $\KL(\rho^2\|\pi^2) = 2\KL(\rho\|\pi)$ \citep[Page 814]{GLL+15}, which gives the factor 2 in front of the first $\KL$ term. The factor 4 in the logarithms comes from a union bound over the bounds on $\E_{\rho^2}[L(h,h')]$ and $\E_\rho[L(h)]$.
\end{proof}

We note that both the loss and the tandem loss are Bernoulli random variables, and for Bernoulli random variables the PAC-Bayes-kl inequality is tighter than the PAC-Bayes-Bennett \citep{TS13}. However, the empirical counterpart of the expected $\mu$-tandem loss is the empirical $\mu$-tandem loss
\[
\hat L_\mu(h,h',S) = \frac{1}{n}\sum_{i=1}^n (\1[h(X_i)\neq Y_i] - \mu)(\1[h'(X_i) \neq Y_i] - \mu),
\]
and the $\mu$-tandem losses are not Bernoulli. Therefore, we use the PAC-Bayes-Bennett inequality, which provides an advantage if the variance of the $\mu$-tandem losses happens to be small. The expected and empirical variance of the $\mu$-tandem losses of a pair of hypotheses $h$ and $h'$ are, respectively, defined by
\begin{align*}
\Var_\mu(h,h') &= \E_D[\lr{(\1[h(X)\neq Y] - \mu)(\1[h'(X)\neq Y] - \mu) - L_\mu(h,h')}^2],\\
\hat \Var_\mu(h,h',S) &= \frac{1}{n-1}\sum_{i=1}^n \lr{(\1[h(X_i)\neq Y_i] - \mu)(\1[h'(X_i)\neq Y_i] - \mu) - \hat L_\mu(h,h',S)}^2.
\end{align*}
The empirical variance $\hat \Var_\mu(h,h',S)$ is an unbiased estimate of $\Var_\mu(h,h')$.

Since the PAC-Bayes-Bennett inequality is stated in terms of the oracle variance $\E_\rho[\tilde \V(h)]$, we use the result by \citet[Equation (15)]{TS13} to bound it in terms of the empirical variance. For a general loss function $\tilde\ell(\cdot,\cdot)$ (not necessarily within $[0,1]$), we define the empirical variance of the loss of $h$ by $\hat{\tilde \V}(h,S)=\frac{1}{n-1}\sum_{i=1}^n(\tilde\ell(h(X_i),Y_i)-\tilde L(h))^2$. We recall that $\tilde L$, $\tilde \V$, and $\hat{\tilde L}$ were defined above Theorem~\ref{thm:PBBennett}. We note that the result of \citeauthor{TS13} assumes that the losses are bounded in the $[0,1]$ interval. Rescaling to a general range introduces the squared range factor $c^2$ in front of the last term in the inequality below, since scaling a random variable by $c$ scales the variance by $c^2$.
\begin{theorem}[\citealp{TS13}]
\label{thm:var-bound}
Let $\tilde \ell(\cdot,\cdot)$ be an arbitrary bounded loss function and let $c$ be the length of the loss range. Then for any distribution $\pi$ on $\HH$ that is independent of $S$, any $\lambda\in\lr{0,\frac{2(n-1)}{n}}$, and any $\delta\in(0,1)$, with probability at least $1-\delta$ over a random draw of the sample $S$, for all distributions $\rho$ on $\HH$ simultaneously:
\[
\E_\rho[\tilde \V(h)]\leq \frac{\E_\rho[\hat{\tilde \V}(h,S)]}{1-\frac{\lambda n}{2(n-1)}} + \frac{c^2\lr{\KL(\rho\|\pi)+\ln\frac{1}{\delta}}}{n\lambda\lr{1-\frac{\lambda n}{2(n-1)}}}.
\]
\end{theorem}
We note that, similar to the PAC-Bayes-Bennett inequality, but in contrast to the PAC-Bayes-$\lambda$ inequality, the inequality above holds for a fixed value of $\lambda$ and in case of optimization over $\lambda$ a union bound has to be applied.

The last thing that is left is to bound the length of the range of $\mu$-tandem losses defined in equation \eqref{eq:ell-mu}.
\begin{lemma}
\label{lem:range}
For $\mu < 0.5$ we have that the length of the range of $\ell_\mu(\cdot,\cdot,\cdot)$ is $K_\mu=\max\{1-\mu,1-2\mu\}$.
\end{lemma}
A proof is provided in Appendix~\ref{app:range}. Taking together the results of Theorems~\ref{thm:MVBen}, \ref{thm:PBBennett}, \ref{thm:var-bound}, and Lemma~\ref{lem:range} we obtain the following result.
\begin{theorem}\label{thm:second-mu-bound}
For any parameter grid $\{\gamma_1,\dots,\gamma_{k_\gamma}\}$ and $\{\lambda_1,\dots,\lambda_{k_\lambda}\}$, where $\gamma_i > 0$ for all $i$ and $\lambda_i\in\lr{0,\frac{2(n-1)}{n}}$ for all $i$, any distribution $\pi$ on $\HH$ that is independent of $S$, and any $\delta\in(0,1)$, with probability at least $1-\delta$ over a random draw of $S$, for all values of $\mu < 0.5$, all distributions $\rho$ on $\HH$, and all values of $\gamma$ and $\lambda$ in the parameter grid simultaneously:
\begin{multline*}
L(\MV_\rho) \leq \frac{1}{(0.5-\mu)^2}\Bigg(\E_{\rho^2}[\hat L_\mu(h,h',S)] + \frac{2\KL(\rho\|\pi) + \ln\frac{2 k_\gamma k_\lambda}{\delta}}{\gamma n}\\
+\frac{\phi(\gamma K_\mu)}{\gamma K_\mu^2}\lr{\frac{\E_{\rho^2}[\hat \V_\mu(h,h',S)]}{1-\frac{\lambda n}{2(n-1)}} + \frac{K_\mu^2\lr{2\KL(\rho\|\pi)+\ln \frac{2 k_\gamma k_\lambda}{\delta}}}{n\lambda \lr{1-\frac{\lambda n}{2(n-1)}}}}\Bigg).
\end{multline*}
\end{theorem}
\begin{proof}
The result follows by reverse substitution of the result of Lemma~\ref{lem:range} into Theorem~\ref{thm:var-bound}, then into Theorem~\ref{thm:PBBennett}, and finally into Theorem~\ref{thm:MVBen}. Since $\KL(\rho^2\|\pi^2) = 2\KL(\rho\|\pi)$, we have factor 2 in front of the $\KL$ terms. The factor $2 k_\gamma k_\lambda$ comes from a union bound over the parameter grid and the bounds in Theorems~\ref{thm:PBBennett} and \ref{thm:var-bound}.
\end{proof}

\section{Experiments}
\label{sec:exper}

We start with a simulated comparison of the oracle bounds and then present an empirical evaluation on real data. The python source code for replicating the experiments is available at Github\footnote{\url{https://github.com/StephanLorenzen/MajorityVoteBounds}}.

\subsubsection*{Comparison of the oracle bounds}

Figure~\ref{fig:M2vsCC} depicts a comparison of the second order oracle bound based on the Chebyshev-Cantelli inequality (Theorems~\ref{thm:Cbound}, \ref{thm:MV} and \ref{thm:MVBen}, which, as oracle bounds, are equivalent) and the second order oracle bound based on the second order Markov's inequality (Theorem~\ref{thm:tandembound}). 
\begin{wrapfigure}{r}{0.4\textwidth}
    \centering
    \vspace{-.25cm}
    \includegraphics[width=.4\textwidth]{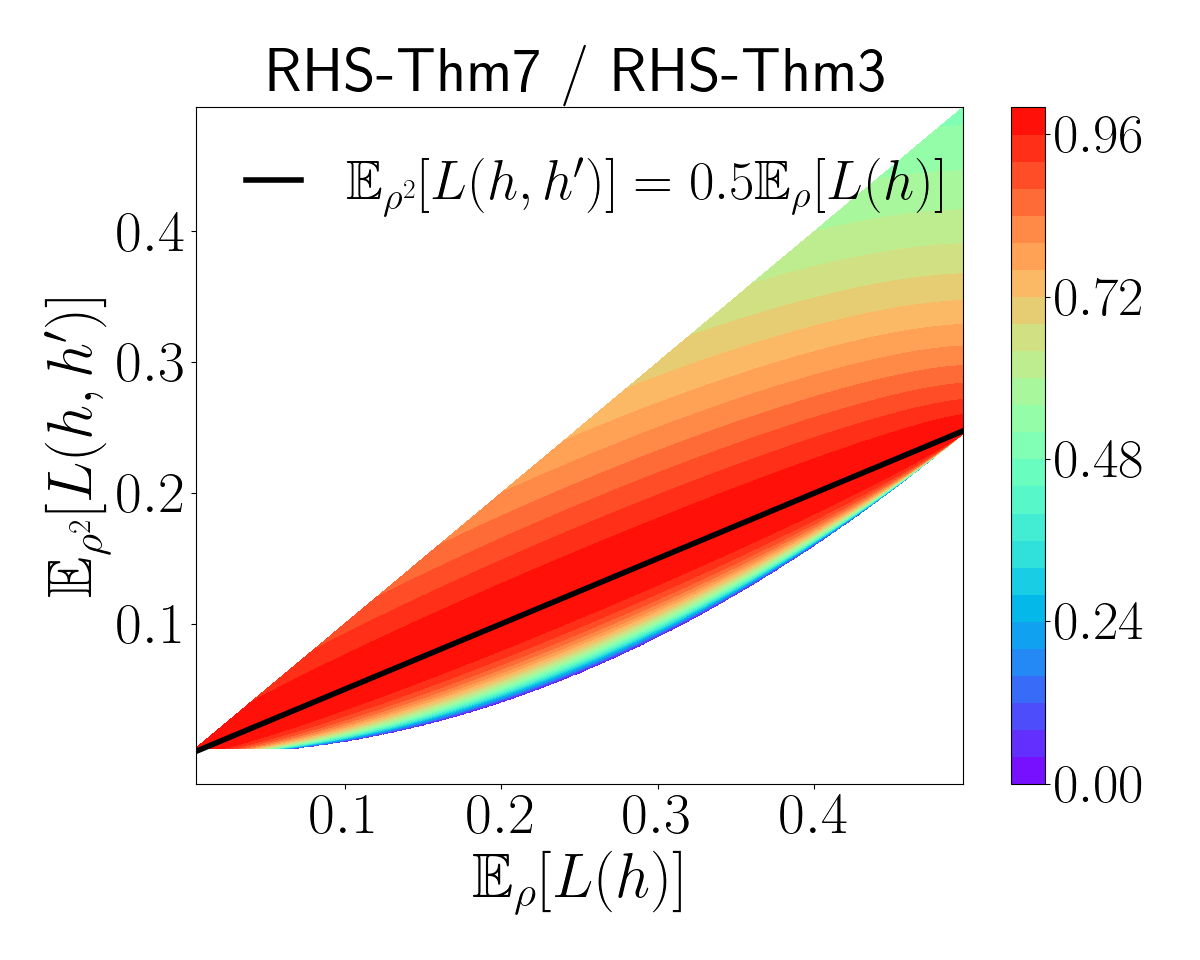}
    \caption{Theorem~\ref{thm:MV} vs.\ Theorem~\ref{thm:tandembound}}
    \vspace{-.5cm}
    \label{fig:M2vsCC}
\end{wrapfigure}
We plot the ratio of the right hand side of the bound in Theorem~\ref{thm:MV} for the optimal value $\mu^*=\E_\rho[L(h)]-\frac{\E_{\rho^2}[L(h,h')]-\E_\rho[L(h)]^2}{0.5 - \E_\rho[L(h)]}$ to the value of the right hand side of the bound in Theorem~\ref{thm:tandembound}. A simple calculation shows that if $\E_{\rho^2}[L(h,h')] = 0.5 \E_\rho[L(h)]$, then $\mu^*=0$, which recovers the bound in Theorem~\ref{thm:tandembound}. The line $\E_{\rho^2}[L(h,h')] = 0.5 \E_\rho[L(h)]$ is shown in black in Figure~\ref{fig:M2vsCC}. We also note that $\E_\rho[L(h)]^2 \leq \E_{\rho^2}[L(h,h')] \leq \E_\rho[L(h)]$, which defines the feasible region in Figure~\ref{fig:M2vsCC}. 
Whenever $\E_{\rho^2}[L(h,h')] \neq 0.5 \E_\rho[L(h)]$ the Chebyshev-Cantelli inequality provides an improvement over second order Markov's inequality. The region above the black line, where $\E_{\rho^2}[L(h,h')] > 0.5 \E_\rho[L(h)]$, is the region of high correlation of errors and in this case majority vote yields little improvement over individual classifiers. In this region the first order oracle bound is tighter than the second order oracle bounds (see Appendix~\ref{app:FOvsSO}). 
The region below the black line, where $\E_{\rho^2}[L(h,h')] < 0.5 \E_\rho[L(h)]$, is the region of low correlation of errors. In this region the second order oracle bounds are tighter than the first order oracle bound. Note that the potential for improvement below the black line is much higher than above it.

\subsubsection*{Empirical evaluation on real datasets}

We studied the empirical performance of the bounds using standard random forest \citep{Bre01} and a combination of heterogeneous classifiers on a subset of data sets from UCI and LibSVM repositories \citep{UCI,libsvm}. An overview of the data sets is given in Appendix~\ref{app:datasets}. The number of points varied from 3000 to 70000 with dimensions $d<1000$. For each data set, we set aside 20\% of the data for the test set $\testset$ and used the remaining data $S$ for ensemble construction, weight optimization and bound evaluation.
We evaluate the classifiers and bounds obtained by minimizing the tandem bound $\TND$ \citep[Theorem 9]{MLIS20}, which is the empirical bound on the oracle tandem bound in Theorem~\ref{thm:tandembound}, the Chebyshev-Cantelli bound with TND empirical loss estimate bound $\CMUTND$ (Theorem~\ref{thm:first-mu-bound}), and the Chebyshev-Cantelli bound with PAC-Bayes-Bennett loss estimate bound $\COTND$ (Theorem~\ref{thm:second-mu-bound}).
We made 10 repetitions of each experiment.

\paragraph{Ensemble construction and minimization of the bounds.}
We follow the construction used by \citet{MLIS20}. The idea is to generate multiple random splits of the data set $S$ into pairs of subsets $S=T_h\cup S_h$, such that $T_h\cap S_h=\emptyset$. Each hypothesis is trained on $T_h$ and the empirical loss on $S_h$ provides an unbiased estimate of its expected loss. Note that the splits cannot depend on the data.
For our experiments, we generate these 
splits by bagging, where out-of-bag (OOB) samples $S_h$ provide unbiased estimates of expected losses of individual hypotheses $h$. 
The resulting set of hypotheses produces an ensemble. 
As in the work of \citeauthor{MLIS20}, two modifications are required to apply the bounds: the empirical losses $\hat L(h,S)$ in the bounds are replaced by the validation losses $\hat L(h,S_h)$, and the sample size $n$ is replaced by the minimal validation size $\min_h|S_h|$. For pairs of hypotheses $(h,h')$, we take the overlaps of their validation sets $S_h\cap S_{h'}$ to calculate an unbiased estimate of their tandem loss $\hat L(h,h',S_h\cap S_{h'})$, $\mu$-tandem loss $\hat L_\mu(h,h',S_h\cap S_{h'})$, and the variance of the $\mu$-tandem loss $\hat \Var_\mu(h,h',S_h\cap S_{h'})$, which replaces the corresponding empirical losses in the bounds. The sample size is then replaced by $\min_{h,h'}|S_h\cap S_{h'} |$. The details on bound minimization are provided in Appendix~\ref{app:minimization_details}.


\paragraph{Optimizing weighted random forest.}
In the first experiment we compare $\TND$, $\CMUTND$, and $\COTND$ bounds in the setting studied by \citet{MLIS20}. We take 100 fully grown trees, use the Gini criterion for splitting, and consider $\sqrt{d}$ features in each split. Figure~\ref{fig:RFopt:risks} compares the loss of the random forest on $\testset$ using either uniform weighting $\rho_u$ or optimized weighting $\rho^*$ found by minimization of the three bounds (we exclude the first order bound from the comparison, since it was shown by \citeauthor{MLIS20} that it significantly deteriorates the test error of the ensemble).
While $\CMUTND$ often performs similar to $\TND$, we find that optimizing using $\COTND$ often improves accuracy.
Figure~\ref{fig:RFopt:bounds} compares the tightness of the optimized $\CMUTND$ and $\COTND$ bounds to the optimized $\TND$ bound. The $\CMUTND$ is generally comparable to $\TND$, while $\COTND$ is consistently looser than $\TND$, mainly due to the union bounds. The numerical values for the losses and the bounds can be found in Tables~\ref{tab:RFopt:risks} and \ref{tab:RFopt:bounds} in Appendix~\ref{app:experiments:RFopt}.
\begin{figure}[t]
    \centering
    \begin{subfigure}[b]{0.49\textwidth}
        \centering
        \includegraphics[width=\textwidth]{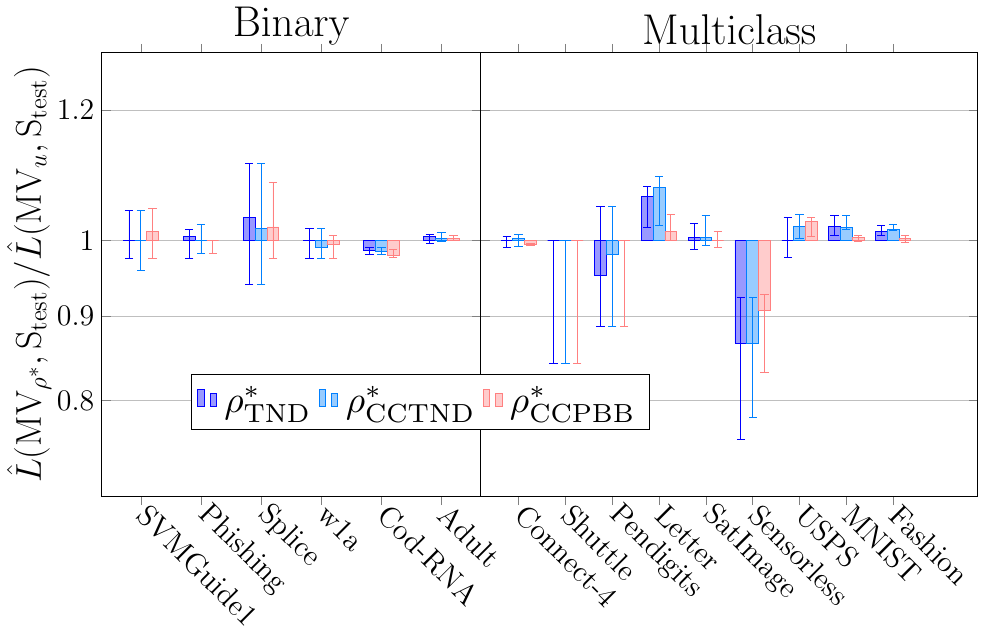}
        \caption{}
        \label{fig:RFopt:risks}
    \end{subfigure}
    \hfill
    \begin{subfigure}[b]{0.49\textwidth}
        \centering
        \includegraphics[width=\textwidth]{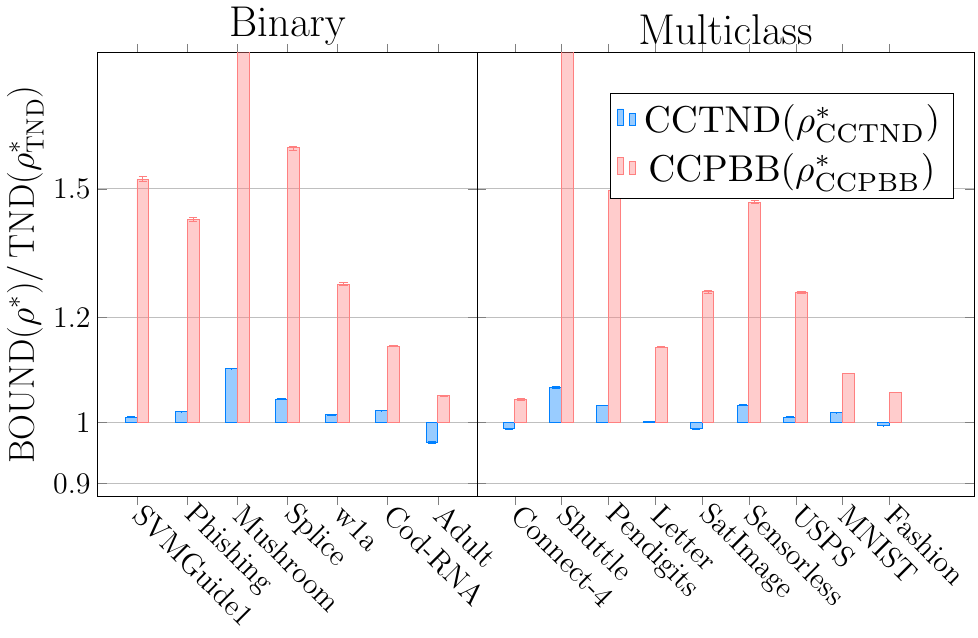}
        \caption{}
        \label{fig:RFopt:bounds}
    \end{subfigure}
    \begin{subfigure}[b]{0.49\textwidth}
        \centering
        \includegraphics[width=\linewidth]{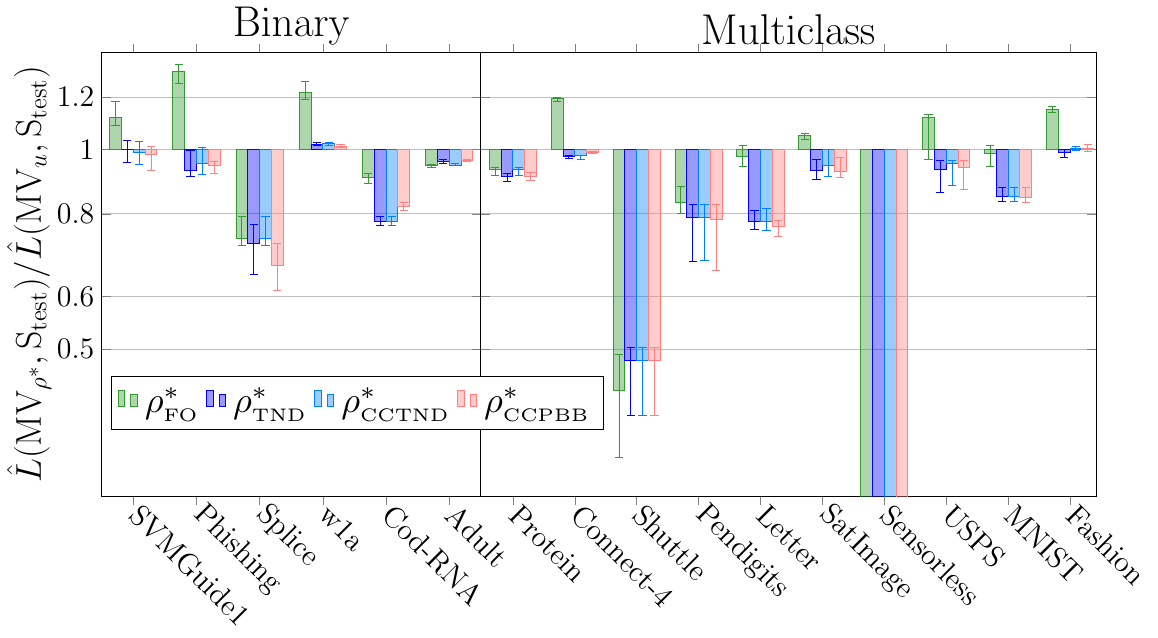}
        \caption{}
        \label{fig:mce:risks}
    \end{subfigure}
    \hfill
    \begin{subfigure}[b]{0.49\textwidth}
        \centering
         \includegraphics[width=\linewidth]{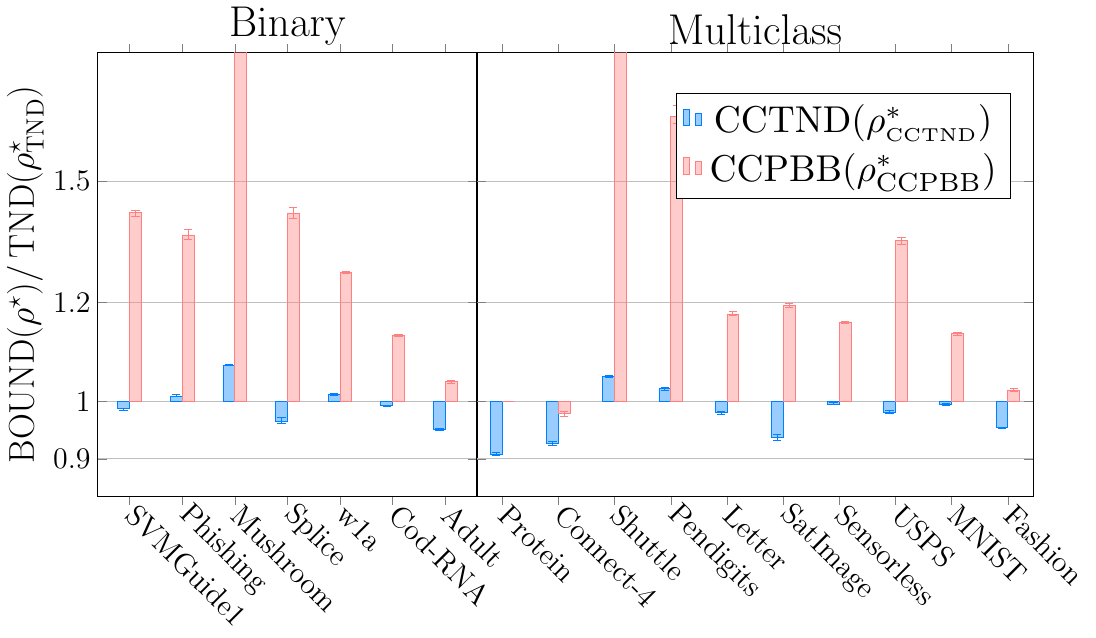}\\
        \caption{}
        \label{fig:mce:bounds}
    \end{subfigure}
    \caption{
    \textbf{(a,b) Optimized weighted random forest.}
    \textbf{(c,d) Ensembles with heterogeneous classifiers.}
    The median, 25\%, and 75\% quantiles of: (a,c) the ratio $\hat L(\MV_{\rho^*}, \testset) / \hat L(\MV_u,\testset)$ of the test loss of the majority vote with optimized weighting $\rho^*$ generated by $\TND$, $\CMUTND$ and $\COTND$ to the test loss of majority vote with uniform weighting, and (b,d) the ratio $\BOUND(\rho^\star)/\TND(\rho^\star_{\scriptscriptstyle{\TND}} )$ of the $\CMUTND$ and $\COTND$ bounds to the $\TND$ bound with the corresponding optimized weighting . The plots are on a logarithmic scale. Values above 1 represent degradation and values below 1 represent improvement.
    Data sets with $L(\MV_u,\testset)= 0$ are left out in (a,c).
    }
    \label{fig:plots}
\end{figure}

\paragraph{Ensembles with heterogeneous classifiers.}
In the second experiment, we consider ensembles of heterogeneous classifiers (Linear Discriminant Analysis, $k$-Nearest Neighbors, Decision Tree, Logistic Regression, and Gaussian Naive Bayes). A detailed description is provided in Appendix \ref{app:experiments:heterogeneous}. Compared to random forests, the variation in performance of ensemble members is larger here. Figure~\ref{fig:mce:risks} compares the ratio of the loss of the majority vote with optimized weighting to the loss of majority vote with uniform weighting on $\testset$ for $\rho^*$ found by minimization of the first order bound ($\FO$), $\TND$, $\CMUTND$, and $\COTND$. The numerical values are given in Table~\ref{tab:mce:risks} in Appendix~\ref{app:experiments:heterogeneous}. We observed that optimizing the $\FO$ tends to improve the ensemble accuracy in some cases but degrade in others. However, $\TND$, $\CMUTND$, and $\COTND$ almost always improve the performance w.r.t.\ the uniform weighting.  Table~\ref{tab:mce:risks} also shows that choosing the best single hypothesis gives almost identical results as optimizing $\FO$. 
Figure~\ref{fig:mce:bounds} compares the tightness of the $\CMUTND$ and $\COTND$ bounds relative to the $\TND$ bound. The numerical values are given in Table~\ref{tab:mce:bounds} in Appendix~\ref{app:experiments:heterogeneous}. In this case, we have that $\CMUTND$ is usually tighter than $\TND$, while $\COTND$ is usually looser than $\TND$ due to the union bounds.


%

\section{Discussion}

We derived an optimization-friendly form of the Chebyshev-Cantelli inequality and applied it 
to derive two new forms of second order oracle bounds for the weighted majority vote. The new oracle bounds bridge between the 
C-bounds \citep{GLL+15} and the tandem bound \citep{MLIS20} and take the best of both:  
the tightness of the Chebyshev-Cantelli inequality 
and the optimization and estimation convenience of the tandem bound. 
We also derived the PAC-Bayes-Bennett inequality, 
improving on the PAC-Bayes-Bernstein inequality of \citet{SLCB+12}.

Our paper opens several directions for future research. One of them is a better treatment of parameter search in parametric bounds that would give tighter bounds than a union bound over a grid. It would also be interesting to find other applications for the new form of Chebyshev-Cantelli inequality and the PAC-Bayes-Bennett inequality.

\begin{ack}
We thank the anonymous reviewers, as well as Tim van Erven, Wouter Koolen, and Peter Gr{\" u}nwald for their constructive feedback, references, and for pointing out that there is no need in a union bound over a grid of $\mu$ in Theorems~\ref{thm:first-mu-bound} and \ref{thm:second-mu-bound}, and that negative $\mu$ in Theorem~\ref{thm:first-mu-bound} requires a separate treatment.

This project has received funding from European Union’s Horizon 2020 research and innovation
programme under the Marie Skłodowska-Curie grant agreement No 801199. YW and YS acknowledge support by the Independent Research Fund Denmark, grant number 0135-00259B. SSL acknowledges funding by the Danish Ministry of Education and Science, Digital Pilot Hub and Skylab Digital. CI acknowledges support by the Villum Foundation through the project Deep Learning and Remote Sensing for Unlocking Global Ecosystem Resource Dynamics (DeReEco). AM is funded by the Spanish Ministry of Science, Innovation and Universities and by the regional government of Andaluc\'ia, grant numbers PID2019-106758GB-C32 and P20-00091, respectively, and by FEDER funds. 
\end{ack}



\bibliography{bibliography}
\bibliographystyle{plainnat}

\clearpage

\pagebreak

\appendix

\section{A proof of the PAC-Bayes-Bennett inequality (Theorem~\ref{thm:PBBennett}) and  a comparison with the PAC-Bayes-Bernstein inequality}
\label{app:Bennett}

In this section we provide a proof of Theorem~\ref{thm:PBBennett} and a numerical comparison with the PAC-Bayes-Bernstein inequality. The proof is based on the standard change of measure argument. We use the following version by \citet{TS13}.

\begin{lemma}[PAC-Bayes Lemma]
\label{lem:PAC-Bayes}
For any function $f_n:\HH\times ({\cal X}\times{\cal Y})^n\to \R$ and for any distribution $\pi$ on $\HH$, such that $\pi$ is independent of $S$, with probability at least $1-\delta$ over a random draw of $S$, for all distributions $\rho$ on $\HH$ simultaneously:
\[
\E_\rho[f_n(h,S)] \leq \KL(\rho\|\pi) + \ln \frac{1}{\delta} + \ln\E_\pi[\E_{S'}[e^{f_n(h,S')}]].
\]
\end{lemma}

The second ingredient is Bennett's lemma, which is a bound on the moment generating function used in the proof of Bennett's inequality. Since we are unaware of a reference, we provide a proof below, which is essentially an intermediate step in the proof of Bennett's inequality \citep[Theorem 2.9]{BLM13}.
\begin{lemma}[Bennett's Lemma]
\label{lem:Bennett}
Let $b>0$ and let $Z_1,\dots,Z_n$ be i.i.d.\ \emph{zero-mean} random variables with finite variance, such that $Z_i \leq b$ for all $i$. Let $M_n = \sum_{i=1}^n Z_i$ and $V_n = \sum_{i=1}^n \E[Z_i^2]$. Let $\phi(u)=e^u-u-1$. Then for any $\lambda>0$:
\[
\E[e^{\lambda M_n - \frac{\phi(b\lambda)}{b^2} V_n}]\leq 1.
\]
\end{lemma}

\begin{proof}
Since $u^{-2}\phi(u)$ is a non-decreasing function of $u\in\R$ (where at zero we continuously extend the function), for all $i\in[n]$ and $\lambda>0$ we have
\begin{equation*}
    e^{\lambda Z_i}-\lambda Z_i -1 \leq Z_i^2\frac{\phi(b\lambda)}{b^2},
\end{equation*}
which implies
\begin{equation*}
    \E[e^{\lambda Z_i}]\leq  1+ \lambda\E[Z_i] + \frac{\phi(b\lambda)}{b^2} \E[Z_i^2]\leq e^{\frac{\phi(b\lambda)}{b^2} \E[Z_i^2]},
\end{equation*}
where the second inequality uses the assumption that $\E[Z_i]=0$ and the fact that $1+x\leq e^x$ for all $x\in \R$. By the above inequality and independence of the random variables,
\begin{align*}
    \E[e^{\lambda M_n-\frac{\phi(b\lambda)}{b^2}V_n}] &= \E[\prod_{i=1}^n e^{\lambda Z_i - \frac{\phi(b\lambda)}{b^2} \E[Z_i^2] }] = \prod_{i=1}^n \E[e^{\lambda Z_i - \frac{\phi(b\lambda)}{b^2} \E[Z_i^2] }] \leq 1.
\end{align*}
\end{proof}

Now we are ready to prove the theorem.
\begin{proof}[Proof of Theorem~\ref{thm:PBBennett}]
We take $f_n(h,S)=\gamma n\lr{\tilde L(h)- \hat{\tilde L}(h,S)} - \frac{\phi(\gamma b)}{b^2} n \tilde \V(h)$. Then by Lemma~\ref{lem:Bennett} we have $\E_S[e^{f_n(h,S)}]\leq 1$. By plugging this into Lemma~\ref{lem:PAC-Bayes}, normalizing by $\gamma n$, and changing sides, we obtain the result.
\end{proof}

\subsubsection*{Numerical comparison of the PAC-Bayes-Bennett and PAC-Bayes-Bernstein bound}
Figure~\ref{fig:Bennett_vs_Bernstein} provides a numerical comparison of the PAC-Bayes-Bennett and PAC-Bayes-Bernstein inequalities (Theorem~\ref{thm:PBBennett} and Theorem~7 by \citet{TS13}).
\begin{figure}[h]
    \centering
    \begin{subfigure}[b]{0.49\textwidth}
        \centering
        \includegraphics[width=\linewidth]{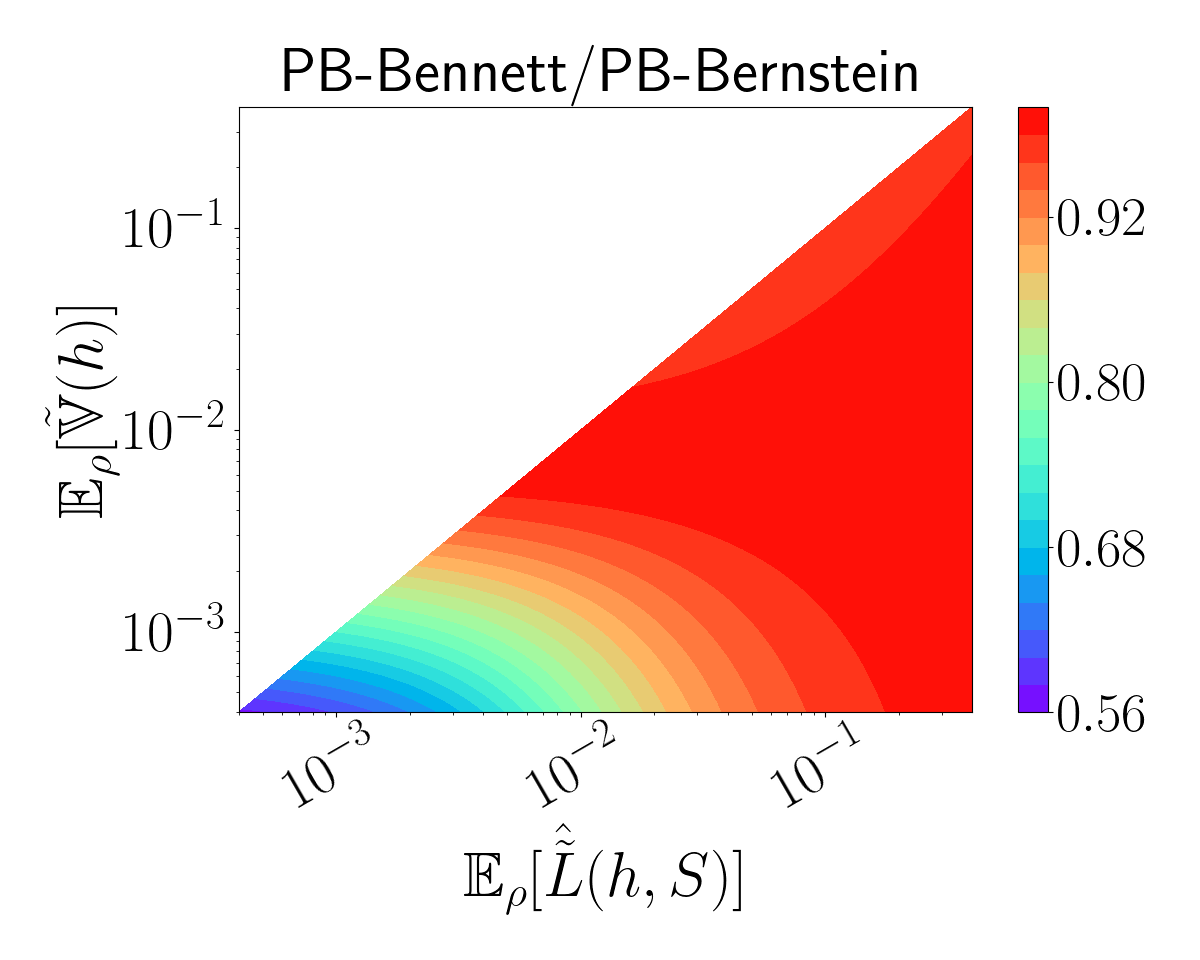}
        \caption{$n=1000$}
        \label{fig:Ben_vs_Bernstein_1000}
    \end{subfigure}
    \hfill
    \begin{subfigure}[b]{0.49\textwidth}
        \centering
         \includegraphics[width=\linewidth]{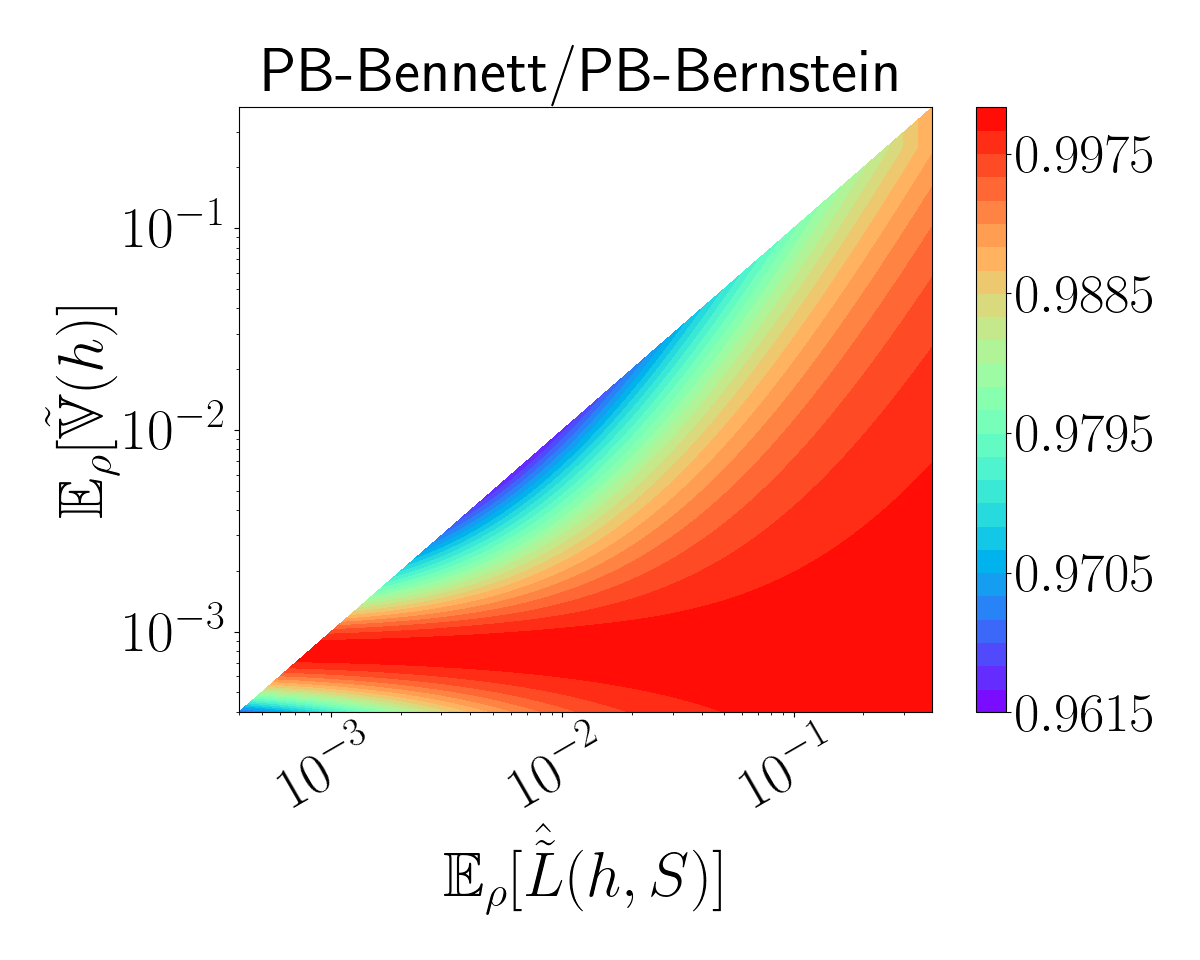}\\
        \caption{$n=10000$}
        \label{fig:Ben_vs_Berntein_10000}
    \end{subfigure}
    \caption{The ratio of PAC-Bayes Bennett to PAC-Bayes Bernstein bound as a function of $\E_\rho[\hat{\tilde L}(h,S)]$ and $\E_\rho[\tilde \V(h)]$. We set $\KL(\rho\|\pi)=5$ and $\delta=0.05$. The value of $n$ is provided in the captions of the subfigures.}
    \label{fig:Bennett_vs_Bernstein}
\end{figure}

\section{Proof of Lemma~\ref{lem:range}}
\label{app:range}

\begin{proof}
Recall that
\begin{equation*}
    \ell_\mu(h(X),h'(X),Y)= (\1[h(X)\neq Y]-\mu)(\1[h'(X)\neq Y]-\mu) \in \lrc{(1-\mu)^2,-\mu(1-\mu),\mu^2}.
\end{equation*}
For $\mu<0.5$, we have $-\mu(1-\mu)<(1-\mu)^2$ and $\mu^2<(1-\mu)^2$.  Therefore, $\ell_\mu(h(X),h'(X),Y)\leq (1-\mu)^2$.

Furthermore, for $\mu < 0$ we have $\mu^2<-\mu(1-\mu)$, and for $\mu > 0$ we have $-\mu(1-\mu)\leq \mu^2$. Therefore, for $\mu<0.5$ we have
$\ell_\mu(h(X),h'(X),Y)\geq \min\{-\mu(1-\mu), \mu^2 \}$.

By combining the upper and the lower bound, we obtain
\begin{align*}
    K_\mu & = (1-\mu)^2-\min\{-\mu(1-\mu), \mu^2 \}\\
    &=\max \{(1-\mu)^2-(-\mu(1-\mu)), (1-\mu)^2-\mu^2 \}\\
    &=\max\{1-\mu, 1-2\mu \}.
\end{align*}
\end{proof}

\section{Comparison of the first and second order oracle bounds}
\label{app:FOvsSO}

In this section we show that if $\E_\rho[L(h)] < 0.5$ and $\E_{\rho^2}[L(h,h')] > 0.5 \E_\rho[L(h)]$, then the first order oracle bound is tighter than the second order oracle bounds, and if $\E_\rho[L(h)] < 0.5$ and $\E_{\rho^2}[L(h,h')] < 0.5 \E_\rho[L(h)]$, then it is the other way around.

For comparison of the first order oracle bound $L(\MV_\rho)\leq 2\E_\rho[L(h)]$ vs.\ the second order oracle tandem bound $L(\MV_\rho)\leq 4\E_{\rho^2}[L(h,h')]$ the statement above is evident.

For the second order oracle bounds based on the  Chebyshev-Cantelli inequality we have
\begin{align*}
    \frac{\E_{\rho^2}[L(h,h')] - \E_\rho[L(h)]^2}{0.25 + \E_{\rho^2}[L(h,h')] - \E_\rho[L(h)]} ~~~~~&\text{vs.}~~~~~2\E_\rho[L(h)],\\
    \E_{\rho^2}[L(h,h')] - \E_\rho[L(h)]^2 ~~~~~&\text{vs.}~~~~~ 0.5\E_\rho[L(h)] + 2\E_\rho[L(h)]\E_{\rho^2}[L(h,h')] - 2 \E_\rho[L(h)]^2,\\
    \E_{\rho^2}[L(h,h')] (1- 2 \E_\rho[L(h)]) ~~~~~&\text{vs.}~~~~~ 0.5 \E_\rho[L(h)](1 - 2 \E_\rho[L(h)]),\\
    \E_{\rho^2}[L(h,h')] ~~~~~&\text{vs.}~~~~~ 0.5 \E_\rho[L(h)],
\end{align*}
where under the assumption that $\E_\rho[L(h)] < 0.5$ we can cancel $(1-2\E_\rho[L(h)])$, since it is positive, and the result is again evident.

\section{Minimization of the bounds}\label{app:minimization_details}

In this section we provide technical details on minimization of the bounds in Theorems~\ref{thm:first-mu-bound} and \ref{thm:second-mu-bound}. As most of the other PAC-Bayesian works, we take $\pi$ to be a union distribution over the hypotheses in both cases.
As discussed in Section \ref{sec:exper}, we build a set of data-dependent hypotheses by splitting the data set $S$ into pairs of subsets $S = T_h \cup S_h$, such that $T_h \cap S_h = \emptyset$, training $h$ on $T_h$ and calculating an unbiased loss estimate $\hat L(h, S_h)$ on $S_h$. For tandem losses we compute the unbiased estimates $\hat L(h,h',S_h\cap S_{h'})$ on the intersections of the corresponding sets $S_h$ and $S_{h'}$. 

\subsection{Minimization of the bound in Theorem \ref{thm:first-mu-bound}} \label{app:minimization_cmutnd}
The adjustment of the bound from Theorem~\ref{thm:first-mu-bound} to this construction is for $\mu\geq 0$:
\begin{multline*}
L(\MV_\rho)\leq \frac{1}{(0.5-\mu)^2}\bigg[\frac{\E_{\rho^2}[\hat L(h,h',S_h\cap S_{h'})]}{1 - \frac{\lambda}{2}} + \frac{2\KL(\rho\|\pi) + \ln(4 \sqrt m/\delta)}{\lambda\lr{1-\frac{\lambda}{2}}m}\\ 
- 2\mu\lr{\lr{1 - \frac{\gamma}{2}}\E_\rho[\hat L(h,S_h)] - \frac{\KL(\rho\|\pi) + \ln(4 \sqrt n/\delta)}{\gamma n}} + \mu^2\bigg],
\end{multline*}
and for $\mu<0$:
\begin{multline*}
L(\MV_\rho)\leq \frac{1}{(0.5-\mu)^2}\bigg[\frac{\E_{\rho^2}[\hat L(h,h',S_h\cap S_{h'})]}{1 - \frac{\lambda}{2}} + \frac{2\KL(\rho\|\pi) + \ln(4 \sqrt m/\delta)}{\lambda\lr{1-\frac{\lambda}{2}}m}\\ 
- 2\mu\lr{ \frac{\E_{\rho}[\hat L(h,S_h)]}{1 - \frac{\gamma}{2}} + \frac{\KL(\rho\|\pi) + \ln(4\sqrt n/\delta)}{\gamma\lr{1-\frac{\gamma}{2}}n}} + \mu^2\bigg],
\end{multline*}
where $m = \min_{h,h'}|S_h\cap S_{h'}|$ and $n = \min_h |S_h|$. Below we provide the pseudocode and derive update rules for $\mu$, $\lambda$, $\gamma$, and $\rho$ for alternating minimization of this bound.

\begin{algorithm}[H]
	\caption{Minimization of the bound in Theorem~\ref{thm:first-mu-bound}}
	\label{alg:CCTND}
\begin{algorithmic}
	\STATE {\bfseries Input: }{$m, n$, tandem losses $\hat L(h,h',S_h\cap S_{h'})$ for all  $h,h'$, and Gibbs losses $\hat L(h,S_h)$ for all $h$ }\\
	\STATE {\bfseries Initialize: }{$\rho=\pi$ and $\mu=0$}\\
	\WHILE{The improvement of the bound is larger than $10^{-9}$}
		\STATE Compute $\lambda_\rho^*$, the optimal $\lambda$ given $\rho$ \\
		\STATE Compute $\gamma_\rho^*$, the optimal $\gamma$ given $\rho$ and $\mu$ \\
		\STATE Compute the bound using $\rho$, $\mu$, $\lambda_\rho^*$ and $\gamma_\rho^*$ \\
		\STATE Compute new $\mu_\rho^*$, the optimal $\mu$ given $\rho$, $\lambda_\rho^*$ and $\gamma_\rho^*$ \\
		\STATE Update the new distribution $\rho'$ with gradient descent given $\mu$, $\lambda_\rho^*$ and $\gamma_\rho^*$ \\
		\STATE Let $\rho=\rho'$ and $\mu=\mu_\rho^*$
	\ENDWHILE
\end{algorithmic}
\end{algorithm}

\paragraph{Optimal $\lambda$ given $\rho$}
Minimization of the bound with respect to $\lambda$ is identical to minimization of the tandem bound by \citet[Theorem 9]{MLIS20}. \citeauthor{MLIS20} derive the optimal value of $\lambda$:
\begin{equation*}\label{eq:opt_lambda}
    \lambda_\rho^* = \frac{2}{\sqrt{\frac{2m\E_{\rho^2}[\hat L(h,h',S_h\cap S_{h'})]}{2\KL(\rho||\pi)+\ln \frac{4\sqrt{m}}{\delta}}+1}+1}.
\end{equation*}

\paragraph{Optimal $\gamma$ given $\rho$ and $\mu$}
Minimization of the bound with respect to $\gamma$ in the case of $\mu\geq 0$ is analogous to minimization of the bound by \citet[Theorem 10]{MLIS20} with respect to $\gamma$. \citeauthor{MLIS20} derive the optimal value of $\gamma$:

\begin{equation*}\label{eq:opt_gamma_+}
    \gamma_\rho^* = \sqrt{\frac{2\KL(\rho||\pi) + \ln(16n/\delta^2)}{n\E_\rho[\hat L(h,S_h)] }}.
\end{equation*}
On the other hand, the optimal $\gamma$ in the case of $\mu<0$ is analogous to the optimal $\lambda$ above:
\begin{equation*}\label{eq:opt_gamma_-}
    \gamma_\rho^* = \frac{2}{\sqrt{\frac{2n\E_\rho[\hat L(h,S_h)]}{\KL(\rho||\pi)+\ln \frac{4\sqrt{n}}{\delta}}+1}+1}.
\end{equation*}

\paragraph{Optimal $\mu$ given $\rho$}
Given $\rho$, we can compute the optimal $\lambda_\rho^*$ and $\gamma_\rho^*$ by the above formulas. Let \begin{align*}
    U_T(\rho)&:=\frac{\E_{\rho^2}[\hat L(h,h',S_h\cap S_{h'})]}{1 - \frac{\lambda_\rho^*}{2}} + \frac{2\KL(\rho\|\pi) + \ln(4 \sqrt m/\delta)}{\lambda_\rho^*\lr{1-\frac{\lambda_\rho^*}{2}}m},\\
    L_G(\rho)&:=\left\{\begin{array}{cc}
       \lr{1 - \frac{\gamma_\rho^*}{2}}\E_\rho[\hat L(h,S_h)] - \frac{\KL(\rho\|\pi) + \ln(4 \sqrt n/\delta)}{\gamma_\rho^* n},  & \mu\geq 0 \\
       \frac{\E_{\rho}[\hat L(h,S_h)]}{1 - \frac{\gamma_\rho^*}{2}} + \frac{\KL(\rho\|\pi) + \ln(4\sqrt n/\delta)}{\gamma_\rho^*\lr{1-\frac{\gamma_\rho^*}{2}}n},  & \mu<0
    \end{array}\right.
\end{align*}
Then the optimal $\mu$ is
\begin{equation*}
    \mu_\rho^* = \frac{\frac{1}{2}L_G(\rho)-U_T(\rho)}{\frac{1}{2}-L_G(\rho)}.
\end{equation*}

\paragraph{Gradient w.r.t.\ $\rho$ given $\lambda$, $\gamma$ and $\mu$}
Minimization of the bound w.r.t.\ $\rho$ is equivalent to constrained optimization of $f(\rho) = a\E_{\rho^2}[\hat L(h,h',S_h\cap S_{h'})]-2b\E_{\rho}[\hat L(h,S_h)]+2c\KL(\rho||\pi)$, where for $\mu\geq 0$, $a=1/(1-\lambda/2)$, $b=\mu(1-\gamma/2)$ and $c=1/(\lambda(1-\lambda/2)m)+\mu/(\gamma n)$, and for $\mu<0$, $a=1/(1-\lambda/2)$,  $b=\mu/(1-\gamma/2)$, and $c=1/(\lambda(1-\lambda/2)m)-\mu/(\gamma(1-\gamma/2) n)$. The constraint is that $\rho$ is a probability distribution. We optimize $\rho$ by projected gradient descent, where we iteratively take steps in the direction of the negative gradient of $f$ and project the result onto the probability simplex.

We use $\hat L$ to denote the vector of empirical losses and $\hat L_{\operatorname{tnd}}$ to denote the matrix of tandem losses. Let $\nabla f$  denote the gradient of $f$ w.r.t.\ $\rho$ and $(\nabla f)_h$ the $h$-th coordinate of the gradient. We have:
\begin{align*}
    (\nabla f)_h &= 2\left(a \sum_{h'}\rho(h')\hat L(h,h',S_h\cap S_{h'}) - b\hat L(h,S_h) + c\left(1+\ln\frac{\rho(h)}{\pi(h)}\right)\right),\\
    \nabla f &= 2\left(a\hat L_{\operatorname{tnd}}\rho - b\hat L +c\left(1+\ln\frac{\rho}{\pi}\right)\right).
\end{align*}

\paragraph{Gradient descent optimization w.r.t.\ $\rho$}
 To optimize the weighting $\rho$, we
applied iRProp+ for the gradient based optimization, a first order method with adaptive individual
step sizes \citep{IH03,FI18}, until
the bound did not improve for 10 iterations.

\subsection{Minimization of the bound in Theorem~\ref{thm:second-mu-bound}}

We start with the details on construction of the grid of $\mu$, $\lambda$ and $\gamma$.

\subsubsection{The $\mu$ grid for Theorem \ref{thm:second-mu-bound}}
We were unable to find a closed-form solution for minimization of the bound w.r.t.\ $\mu$ and applied a heuristic. Empirically we observed that the bound was quasiconvex in $\mu$ (we were unable to prove that it is always the case) and applied binary search for $\mu$ in the grid. Note that even if we take a grid of $\mu$, we don't need a union bound since the bound holds with high probability for all $\mu$ simultaneously.

We then consider the relevant range of $\mu$. By Theorem \ref{thm:CCmu}, we have $\mu < 0.5$. At the same time, $\mu^* = \frac{0.5\E_\rho[L(h)]-\E_{\rho^2}[L(h,h)]}{0.5-\E_\rho[L(h)]}$, and in Section~\ref{sec:exper} we have shown that the primary region of interest is where $\E_{\rho^2}[L(h,h')] < 0.5 \E_\rho[L(h)]$, which corresponds to $\mu^* > 0$. However, since $\E_{\rho^2}[L(h,h)]$ and $\E_\rho[L(h)]$ are unobserved and we use an upper bound for the first and a lower bound for the second instead, we take a broader range of $\mu$. By making a mild assumption that the upper bound for the tandem loss $\E_{\rho^2}[L(h,h')]$ is at most 0.25 and the lower bound for the Gibbs loss $\E_\rho[L(h)]$ is at most 0.5, we have $\mu \in [-0.5,0.5)$.
We take 400 uniformly spaced points in the selected range for the CCPBB bound.

\subsubsection{The $\lambda$ grid for Theorem \ref{thm:second-mu-bound}}

The parameter $\lambda$ comes from Theorem~\ref{thm:var-bound}. The theorem is identical to the result by \citet[Equation (15)]{TS13}, except rescaling, but rescaling happens on top of the bound and has no effect on the $\lambda$-grid. Therefore, we use the grid proposed by \citeauthor{TS13}. Namely, we take 
\[
\lambda_i = c_1^{i-1} \frac{2(n-1)}{n}\left(\sqrt{\frac{n-1}{\ln(1/\delta_1)}+1}+1\right)^{-1}
\]
for $i\in\{1,\dots,k_\lambda\}$ and \[k_\lambda=\left\lceil\frac{1}{\ln c_1}\ln\left(\frac{1}{2}\sqrt{\frac{n-1}{\ln(1/\delta_1)}+1}+\frac{1}{2}\right) \right\rceil.\]
In the experiments we took $c_1=1.05$ and $\delta_1=\delta/2$.

\subsubsection{The $\gamma$ grid for Theorem \ref{thm:second-mu-bound}}
The parameter $\gamma$ comes from Theorem~\ref{thm:PBBennett}. By taking the first two derivatives we can verify that for a fixed $\rho$ the PAC-Bayes-Bennett bound is convex in $\gamma$ and at the minimum point the optimal value of $\gamma$ satisfies
\[
e^{(\gamma_\rho^* b - 1)}\lr{\gamma_\rho^* b - 1} = \frac{1}{e}\lr{\frac{b^2\lr{\KL(\rho\|\pi) + \ln \frac{1}{\delta_2}}}{n \E_\rho[\tilde \V(h)]} - 1}.
\]
Thus, the optimal value of $\gamma$ is given by 
\[
\gamma_\rho^* = \frac{1}{b}\lr{ W_0\lr{\frac{1}{e}\lr{\frac{b^2\lr{\KL(\rho\|\pi) + \ln \frac{1}{\delta_2}}}{n \E_\rho[\tilde \V(h)]} - 1}} + 1},
\]
where $W_0$ is the principal branch of the Lambert W function, which is defined as the inverse of the function $f(x) = x e^x$.

In order to define a grid for $\gamma$ we first determine the relevant range for $\gamma_\rho^*$. We note that the variance $\E_\rho[\tilde \V(h)]$ is estimated using Theorem~\ref{thm:var-bound}, which assumes that the length of the range of the loss $\tilde \ell(\cdot,\cdot)$ is $c$. The loss range provides a trivial upper bound on the variance $
\E_\rho[\tilde \V(h)] \leq \frac{c^2}{4}$. At the same time, we have $\lambda\lr{1- \frac{\lambda n}{2(n-1)}}\leq \frac{n-1}{2n}$ (it is a downward-pointing parabola) and, therefore, the right hand side of the bound in Theorem~\ref{thm:var-bound} is at least the value of its second term, which is at least $\frac{2c^2\ln \frac{1}{\delta_1}}{n-1}$, since $\KL(\rho\|\pi)\geq 0$. Thus, we obtain that the estimate of $\E_\rho[\tilde \V(h)]$ is in the range $\lrs{\frac{2c^2\ln \frac{1}{\delta_1}}{n-1},\frac{c^2}{4}}$. We use $\Vmin=\frac{2c^2\ln \frac{1}{\delta_1}}{n-1}$ to denote the lower bound of this range.

Since $W_0(\cdot)$ is a monotonically increasing function, $\KL(\rho\|\pi)\geq 0$, and the estimate of $\E_\rho[\tilde \V(h)]$ is at most $\frac{c^2}{4}$, we obtain that $\gamma_\rho^*$ satisfies
\begin{align*}
    \gamma_\rho^* &= \frac{1}{b}\lr{ W_0\lr{\frac{1}{e}\lr{\frac{b^2\lr{\KL(\rho\|\pi) + \ln \frac{1}{\delta_2}}}{n \E_{\rho^2}[\tilde \V(h)]} - 1}} + 1}\\
    &\geq \frac{1}{b}\lr{W_0\lr{\frac{1}{e}\lr{\frac{4b^2}{nc^2}\ln \frac{1}{\delta_2} -1}}+1}\stackrel{def}{=}\gamma_{\min}.
\end{align*}

For an upper bound we observe that since $\E_\rho[\tilde L(h)] - \E_\rho[\hat{\tilde L}(h,S)]$ is trivially bounded by $b$, the bound in Theorem~\ref{thm:var-bound} is only interesting if it is smaller than $b$ and, in particular, $\frac{\phi(\gamma b)}{\gamma b^2}\E_\rho[\tilde \V(h)] \leq b$. 
This gives
\[
b \geq \frac{\phi(\gamma b)}{\gamma b^2}\E_\rho[\tilde \V(h)] \geq \frac{\phi(\gamma b)}{\gamma b^2} \Vmin.
\]
Thus, $\gamma$ should satisfy
\[
\phi(\gamma b) \leq \frac{\gamma b^3}{\Vmin},
\]
which gives that the maximal value of $\gamma$, denoted $\gamma_{max}$, is the positive root of 
\[
H(\gamma) = e^{\gamma b} - \gamma b\lr{1+\frac{b^2}{\Vmin}} - 1 = 0.
\]

Let
$\alpha =  \left( 1 + b^2/\Vmin \right)^{-1}\in(0,1)$, and $x = - \gamma b - \alpha$.
Then the above problem is equivalent to finding the root of $f(x)=xe^x-d$ for $d=-\alpha e^{-\alpha}$, which can again be solved by applying the Lambert W function. Since for $\alpha\in(0,1)$, we have $d\in(-1/e, 0)$, which indicates that there are two roots \citep{CGHJK96}. We denote the root greater than $-1$ as $W_0(d)$ and the root less than $-1$ as $W_{-1}(d)$.
It is obvious that $W_0(d)=-\alpha$. However, $W_0(d)$ is not the desired solution, since for $b>0$, $x=-\alpha$ implies $\gamma=0$, but we assume $\gamma>0$. Hence, $W_{-1}(d)$ is the desired root, which gives the corresponding $\gamma = -\frac{1}{b}(W_{-1}(d)+\alpha)>0$. Thus, we obtain
\[
\gamma_{max}= -\frac{1}{b}\left(W_{-1}\left(-\frac{1}{1+\frac{b^2}{\Vmin} } \cdot e^{-\frac{1}{1+\frac{b^2}{\Vmin} }} \right) + \frac{1}{1+\frac{b^2}{\Vmin} } \right).
\]

We construct the grid by taking $\gamma_i = c_2^{i-1}\gamma_{\min}$ for $i\in\{1,\dots,k_\gamma\}$, were $k_\gamma = \lceil \ln(\gamma_{max}/\gamma_{\min})/\ln c_2\rceil$. In the experiments we took $c_2=1.05$, and $\delta_1=\delta_2=\delta/2$.

\subsubsection{Minimization of the bound}
\label{app:minimization_cotnd}
The adjustment of the bound in Theorem~\ref{thm:second-mu-bound} to our hypothesis space construction, as described above, is:
\begin{multline*}
L(\MV_\rho) \leq \frac{1}{(0.5-\mu)^2}\Bigg(\E_{\rho^2}[\hat L_\mu(h,h',S_h\cap S_{h'})] + \frac{2\KL(\rho\|\pi) + \ln\frac{2k}{\delta}}{\gamma n}\\
+\frac{\phi(\gamma K_\mu)}{\gamma K_\mu^2}\lr{\frac{\E_{\rho^2}[\hat \V_\mu(h,h',S_h\cap S_{h'})]}{1-\frac{\lambda n}{2(n-1)}} + \frac{K_\mu^2\lr{2\KL(\rho\|\pi)+\ln \frac{2k}{\delta}}}{n\lambda \lr{1-\frac{\lambda n}{2(n-1)}}}}\Bigg),
\end{multline*}
where $n = \min_{h,h'}| S_h\cap S_{h'}|$ and $k= k_\lambda k_\gamma$.
We minimize the bound without considering $k_\gamma$ and $k_\lambda$ since we define the grid without taking them into consideration. However, we put back $k_\gamma$ and $k_\lambda$ when computing the generalization bound. Thus, when doing the optimization we take $k=1$, but when we compute the bound we take the proper $k= k_\lambda k_\gamma$.

\begin{algorithm}[H]
	\caption{Minimization of the bound in Theorem~\ref{thm:second-mu-bound}}
	\label{alg:CCTND}
\begin{algorithmic}
	\STATE {\bfseries Input: }{$n$, grid of $\mu$ and losses $\1[h(X_i)\neq Y_i]$ for all $(X_i,Y_i)\in S_h$ for all $h$ }\\
	\FOR{$\mu$ selected by the binary search in the grid}
	\STATE {\bfseries Initialize: }{$\rho=\pi$}\\
	\STATE Compute $\hat L_\mu(h,h',S_h\cap S_{h'})$ and $\hat\Var_\mu(h,h',S_h\cap S_{h'})$ for all $h,h'$
	\WHILE{The improvement of the bound for a fixed $\mu$ is larger than $10^{-9}$}
		\STATE Compute $\lambda_{\mu,\rho}^*$, the optimal $\lambda$ given $\rho$ and $\mu$ \\
		\STATE Compute $\gamma_{\mu,\rho}^*$, the optimal $\gamma$ given $\rho$ and $\mu$ \\
		\STATE Apply gradient descent to the bound w.r.t. $\rho$ given $\mu$, $\lambda_{\mu,\rho}^*$ and $\gamma_{\mu,\rho}^*$ \\
	\ENDWHILE
	\STATE Proceed to the next $\mu$ in the grid proposed by the binary search
	\ENDFOR
\end{algorithmic}
\end{algorithm}

\paragraph{Optimal $\lambda$ given $\mu$  and $\rho$}
Given $\mu$ and $\rho$, $\lambda$ can be computed in the same way as in the optimization of Theorem \ref{thm:var-bound}, since the optimization problem is the same, and get
\[
\lambda_{\mu,\rho}^*= \frac{2(n-1)}{n}    \left(\sqrt{\frac{2(n-1)\E_{\rho^2}[\hat\Var_\mu(h,h',S_h\cap S_{h'})]}{K_\mu^2(2\KL(\rho\|\pi)+\ln\frac{2k}{\delta})}+1}+1\right)^{-1}.
\]
In our implementation at every optimization step we took the closest $\lambda$ to the above value from the $\lambda$-grid.

\paragraph{Optimal $\gamma$ given $\mu$ and $\rho$}
Given $\mu$ and $\rho$, the bound for the variance is obtained by plugging in the optimal $\lambda_{\mu,\rho}^*$ computed above. Let
\[
U_\V (\rho,\mu)=\frac{\E_{\rho^2}[\hat \V_\mu(h,h',S_h\cap S_{h'})]}{1-\frac{\lambda_{\mu,\rho}^* n}{2(n-1)}} + \frac{K_\mu^2\lr{2\KL(\rho\|\pi)+\ln \frac{2k}{\delta}}}{n\lambda_{\mu,\rho}^* \lr{1-\frac{\lambda_{\mu,\rho}^* n}{2(n-1)}}}.
\]

Then
\[
\gamma_{\mu,\rho}^* = \frac{1}{K_\mu}\lr{ W_0\lr{\frac{1}{e}\lr{\frac{K_\mu^2\lr{2\KL(\rho\|\pi) + \ln \frac{2k}{\delta}}}{n U_\V(\rho,\mu)} - 1}} + 1},
\]
where $W_0$ is the principal branch of the Lambert W function, which is defined as the inverse of the function $f(x) = x e^x$.
In our implementation at every optimization step we took the closest $\gamma$ to the above value from the $\gamma$-grid.


\paragraph{Gradient w.r.t.\ $\rho$ given $\lambda$, $\gamma$, and $\mu$ }
Optimizing the bound w.r.t.\ $\rho$ is equivalent to constrained optimization of $f(\rho) = \E_{\rho^2}[\hat L_\mu(h,h',S')]+a\E_{\rho^2}[\hat V_\mu(h,h',S')]+2b\KL(\rho||\pi)$, where 
\[
a=\frac{\phi(K_\mu \gamma)}{K_\mu^2\gamma }\frac{1}{1-\frac{n\lambda}{2(n-1)}}, \quad b=\frac{1}{\gamma n}+\frac{\phi(K_\mu \gamma)}{K_\mu^2\gamma }\frac{ K_\mu^2 }{n\lambda(1-\frac{n\lambda}{2(n-1)})},
\]
and the constraint is that $\rho$ must be a probability distribution. We optimize $\rho$ in the same way as presented in Appendix~\ref{app:minimization_cmutnd}. We use $\hat L_\mu$ to denote the matrix of empirical $\mu$-tandem losses and $\hat \V_\mu$ to denote the matrix of empirical variance of the $\mu$-tandem losses. Then, the gradient w.r.t.\ $\rho$ is given by:
\begin{align*}
    (\nabla f)_h &= 2\left( \sum_{h'}\rho(h')(\hat L_\mu(h,h',S')+a\hat\Var\mu(h,h',S')) + b\left(1+\ln\frac{\rho(h)}{\pi(h)}\right)\right),\\
    \nabla f &= 2\left(\hat L_\mu\rho +a\hat \V_\mu\rho +b\left(1+\ln\frac{\rho}{\pi}\right)\right).
\end{align*}
We applied gradient descent in the same way as presented in Appendix \ref{app:minimization_cmutnd}.

\section{Experiments}
\label{app:experiments}

\subsection{Data sets}\label{app:datasets}
As mentioned, we considered data sets from the UCI and LibSVM repositories \citep{UCI,libsvm}, as well as Fashion-MNIST (\dataset{Fashion}) from Zalando Research\footnote{\url{https://research.zalando.com/welcome/mission/research-projects/fashion-mnist/}}. We used data sets with size $3000 \leq N \leq 70000$ and dimension $d \leq 1000$. These relatively large data sets were chosen in order to provide meaningful bounds in the standard bagging setting, where individual trees are trained on $n=0.8N$ randomly subsampled points with replacement and the size of the overlap of out-of-bag sets is roughly $n/9$.  
An overview of the data sets is given in Table~\ref{tab:data_sets}.
\begin{table}[t]
    \centering
    \caption{Data set overview. $c_{\min}$ and $c_{\max}$ denote the minimum and maximum class frequency.}
    \label{tab:data_sets}
    \input{experiments/table_datasets}
\end{table}

For all experiments, we removed patterns with missing entries and made a stratified split of the data set.
For data sets with a training and a test set (\dataset{SVMGuide1},
\dataset{Splice},
\dataset{Adult},
\dataset{w1a},
\dataset{MNIST},
\dataset{Shuttle},
\dataset{Pendigits},
\dataset{Protein},
\dataset{SatImage},
\dataset{USPS})
we combined the training and test sets and shuffled the entire set before splitting.

\subsection{Optimized weighted random forest}
\label{app:experiments:RFopt}

\subsubsection*{Experimental Setting}
This section describes in detail the settings and the results of the empirical evaluation using random forest (RF) majority vote classifiers.

We construct the ensemble from decision trees available in \textit{scikit-learn}. For each data set, an ensemble of 100 trees is trained using bagging (as described in Section~\ref{sec:exper}). For each tree, the Gini criterion is used for splitting and $\sqrt{d}$ features are considered in each split.

We compare the RF using the default uniform weighting $\rho_u$ and the optimized weighting obtained by $\FO$ \citep{TIS16}, $\TND$ \citep{MLIS20}, $\CMUTND$ (Theorem~\ref{thm:first-mu-bound}) and $\COTND$ (Theorem~\ref{thm:second-mu-bound}). Optimization is based on the out-of-bag sets (see Section~\ref{sec:exper}). For each optimized RF, we also compute the optimized bound.

\subsubsection*{Numerical Results}
This section lists the numerical results for the empirical evaluation using RF.
\begin{table}[t]
    \centering
    \caption{Numerical values of the test loss obtained by the RFs with optimized weighting. The smallest loss is highlighted in \textbf{bold}, while the smallest optimized loss is \underline{underlined}.}
    \label{tab:RFopt:risks}
    \begin{adjustbox}{width=\columnwidth,center}
    \input{experiments/RF/risk_table}
    \end{adjustbox}
\end{table}
\begin{table}[t]
    \centering
    \caption{Numerical values of the bounds for the RFs with optimized weighting. The tightest bound is highlighted in \textbf{bold}, while the tightest second-order bound is \underline{underlined}.}
    \label{tab:RFopt:bounds}
    \input{experiments/RF/bound_table}
\end{table}
\begin{table}[t]
    \centering
    \scriptsize
    \caption{Numerical values for Gibbs loss, tandem loss and optimized $\mu$ for the RFs with optimized weighting. We use $\E_\rho[L]$ and $\E_{\rho^2}[L]$ as short-hands for the Gibbs and the tandem loss respectively.}
    \label{tab:RFopt:values}
    \input{experiments/RF/values_table}
\end{table}
Table~\ref{tab:RFopt:risks} provides the numerical values of the test loss obtained by the RFs with uniform weighting and with weighting optimized by $\FO$, $\TND$, $\CMUTND$ and $\COTND$; a visual presentation is given in Figure~\ref{fig:RFopt:risks}. As observed by \cite{MLIS20}, optimization using $\FO$ leads to overfitting, while the second-order bounds does not significantly degrade the performance. Among the second-order bounds, optimizing using $\COTND$ produces the best classifier in most cases.

Table~\ref{tab:RFopt:bounds} provides the numerical values of the optimized bounds; a visual presentation is given in Figure~\ref{fig:RFopt:bounds}.
Table~\ref{tab:RFopt:values} provides the recorded Gibbs loss and tandem loss using the optimized $\rho$. The optimal $\mu$ found is reported for $\CMUTND$ and $\COTND$ as well.

\subsection{Ensemble of multiple heterogeneous classifiers}
\label{app:experiments:heterogeneous}

\subsubsection*{Experimental Setting}
This section describes in detail the settings and the results of the experimental evaluation using an ensemble of multiple heterogeneous classifiers. 

The ensemble is defined by a set of standard classifiers available in \textit{scikit-learn}:
\begin{itemize}
    \item \textbf{Linear Discriminant Analysis}, with default parameters, which includes a singular value decomposition solver.  
    \item Three versions of \textbf{k-Nearest Neighbors}: (i) k=3 and uniform weights (i.e., all points in each neighborhood are weighted equally); (ii) k=5 and uniform weights; and (iii) k=5 where points are weighted by the inverse of their distance. In all cases, it is employed the Euclidean distance. 
    \item \textbf{Decision Tree}, with default parameters, which includes Gini criterion for splitting and no maximum depth. 
    \item \textbf{Logistic Regression}, with default parameters, which includes L2 penalization. 
    \item \textbf{Gaussian Naive Bayes}, with default parameters.  
\end{itemize}

We included three versions of the kNN classifier to test if our bounds could deal with a heterogeneous set of classifiers where some of them are expected to provide highly correlated errors while others are expected to provide much less correlated errors. 

Each of the seven classifiers of the ensemble was learned from a bootstrap sample of the training data set. We did it in the way to be able to compute and optimize our bounds with the out-of-bag-samples as described in Section~\ref{sec:exper}. 

\subsubsection*{Numerical Results}

This section lists the numerical results for the empirical evaluation using ensembles of multiple heterogeneous classifiers.

Table~\ref{tab:mce:risks} provides the numerical values of the test loss obtained by these ensembles with uniform weighting and with weighting optimized by $\FO$, $\TND$, $\CMUTND$ and $\COTND$; a visual presentation is given in Figure~\ref{fig:mce:risks}. In this case, uniform voting is not a competitive weighting scheme. The second-order bounds perform much better than uniform weighting and than the weights computed according to the first-order bound. There is not any clear winner among the second-order bounds. 

Table~\ref{tab:mce:bounds} provides the numerical values of the optimized bounds; a visual presentation is given in Figure~\ref{fig:mce:bounds}. Among the second-order bounds, the CCTND bound is often tighter in this setting.

Table~\ref{tab:mce:values} provides the recorded Gibbs loss and tandem loss using the optimized $\rho$. The optimal $\mu$ found is reported for $\CMUTND$ and $\COTND$ as well.

\begin{table}[t]
    \centering
    \caption{Numerical values of the test loss obtained by ensembles of multiple heterogeneous classifiers with optimized weighting. The smallest loss is highlighted in \textbf{bold}, while the smallest optimized loss is \underline{underlined}.}
    \label{tab:mce:risks}
\begin{adjustbox}{width=\columnwidth,center}
    \input{experiments/mce/risk_table}
\end{adjustbox}
\end{table}
\begin{table}[t]
    \centering
    \caption{Numerical values of the bounds for ensembles of multiple heterogeneous classifiers with optimized weighting. The tightest bound is highlighted in \textbf{bold}, while the tightest second-order bound is \underline{underlined}.}
    \label{tab:mce:bounds}
    \input{experiments/mce/bound_table}
\end{table}
\begin{table}[t]
    \centering
    \scriptsize
    \caption{Numerical values for Gibbs loss, tandem loss and optimized $\mu$ for the heterogeneous classifiers with optimized weighting. We use $\E_\rho[L]$ and $\E_{\rho^2}[L]$ as short-hands for the Gibbs loss and the tandem loss respectively.}
    \label{tab:mce:values}
    \input{experiments/mce/values_table}
\end{table}

\end{document}

%% file: experiments/table_datasets.tex
\begin{tabular}{lcccccc}\toprule
Data set & $N$ & $d$ & $c$ & $c_{\min}$ & $c_{\max}$ & Source\\
\midrule
\dataset{Adult} & 32561 & 123 & 2 & 0.2408 & 0.7592 & LIBSVM (a1a) \\
\dataset{Cod-RNA} & 59535 & 8 & 2 & 0.3333 & 0.6667 & LIBSVM \\
\dataset{Connect-4} & 67557 & 126 & 3 & 0.0955 & 0.6583 & LIBSVM \\
\dataset{Fashion} & 70000 & 784 & 10 & 0.1000 & 0.1000 & Zalando Research \\
\dataset{Letter} & 20000 & 16 & 26 & 0.0367 & 0.0406 & UCI \\
\dataset{MNIST} & 70000 & 780 & 10 & 0.0902 & 0.1125 & LIBSVM \\
\dataset{Mushroom} & 8124 & 22 & 2 & 0.4820 & 0.5180 & LIBSVM \\
\dataset{Pendigits} & 10992 & 16 & 10 & 0.0960 & 0.1041 & LIBSVM \\
\dataset{Phishing} & 11055 & 68 & 2 & 0.4431 & 0.5569 & LIBSVM \\
\dataset{Protein} & 24387 & 357 & 3 & 0.2153 & 0.4638 & LIBSVM \\
\dataset{SVMGuide1} & 3089 & 4 & 2 & 0.3525 & 0.6475 & LIBSVM \\
\dataset{SatImage} & 6435 & 36 & 6 & 0.0973 & 0.2382 & LIBSVM \\
\dataset{Sensorless} & 58509 & 48 & 11 & 0.0909 & 0.0909 & LIBSVM \\
\dataset{Shuttle} & 58000 & 9 & 7 & 0.0002 & 0.7860 & LIBSVM \\
\dataset{Splice} & 3175 & 60 & 2 & 0.4809 & 0.5191 & LIBSVM \\
\dataset{USPS} & 9298 & 256 & 10 & 0.0761 & 0.1670 & LIBSVM \\
\dataset{w1a} & 49749 & 300 & 2 & 0.0297 & 0.9703 & LIBSVM \\
\bottomrule
\end{tabular}

%% file: experiments/RF/risk_table.tex
\begin{tabular}{lccccc}\toprule
Data set & $L(\MV_{u})$ & $L(\MV_{\rho_\lambda})$ & $L(\MV_{\rho_{\TND}})$ & $L(\MV_{\rho_{\CCTND}})$ & $L(\MV_{\rho_{\CCPBB}})$ \\
\midrule
\dataset{SVMGuide1} & \textbf{0.0284 (0.0037)} & 0.0372 (0.0066) & 0.0287 (0.0035) & \underline{0.0286 (0.0036)} & 0.0287 (0.0039) \\
\dataset{Phishing} & \textbf{0.0292 (0.004)} & 0.0371 (0.0073) & \underline{\textbf{0.0292 (0.0036)}} & \underline{\textbf{0.0292 (0.0036)}} & \underline{\textbf{0.0292 (0.004)}} \\
\dataset{Mushroom} & \textbf{0.0 (0.0)} & \underline{\textbf{0.0 (0.0)}} & \underline{\textbf{0.0 (0.0)}} & \underline{\textbf{0.0 (0.0)}} & \underline{\textbf{0.0 (0.0)}} \\
\dataset{Splice} & \textbf{0.0299 (0.009)} & 0.1087 (0.021) & 0.0306 (0.0099) & 0.0309 (0.0092) & \underline{0.0302 (0.01)} \\
\dataset{w1a} & 0.0108 (0.0007) & 0.016 (0.0025) & 0.0108 (0.0006) & \underline{\textbf{0.0107 (0.0006)}} & 0.0108 (0.0006) \\
\dataset{Cod-RNA} & 0.0402 (0.0013) & 0.0712 (0.0064) & \underline{\textbf{0.0395 (0.0014)}} & \underline{\textbf{0.0395 (0.0014)}} & \underline{\textbf{0.0395 (0.0015)}} \\
\dataset{Adult} & \textbf{0.1693 (0.0027)} & 0.1942 (0.0151) & \underline{0.1698 (0.0031)} & 0.1701 (0.003) & \underline{0.1698 (0.0031)} \\
\dataset{Connect-4} & 0.1706 (0.0023) & 0.2803 (0.0165) & 0.1699 (0.002) & 0.1705 (0.0024) & \underline{\textbf{0.1695 (0.0019)}} \\
\dataset{Shuttle} & \textbf{0.0002 (0.0001)} & 0.0003 (0.0002) & \underline{\textbf{0.0002 (0.0001)}} & \underline{\textbf{0.0002 (0.0001)}} & \underline{\textbf{0.0002 (0.0001)}} \\
\dataset{Pendigits} & 0.0096 (0.0023) & 0.0452 (0.0124) & \underline{\textbf{0.0092 (0.0022)}} & 0.0093 (0.0021) & \underline{\textbf{0.0092 (0.0025)}} \\
\dataset{Letter} & \textbf{0.0378 (0.0036)} & 0.1408 (0.0356) & 0.0398 (0.0041) & 0.0402 (0.0042) & \underline{0.0383 (0.0034)} \\
\dataset{SatImage} & \textbf{0.0828 (0.0068)} & 0.1321 (0.0268) & 0.0835 (0.0061) & 0.0839 (0.0062) & \underline{0.0832 (0.006)} \\
\dataset{Sensorless} & 0.0014 (0.0004) & 0.0138 (0.0019) & \underline{\textbf{0.0012 (0.0003)}} & \underline{\textbf{0.0012 (0.0003)}} & \underline{\textbf{0.0012 (0.0003)}} \\
\dataset{USPS} & \textbf{0.0394 (0.0043)} & 0.1325 (0.0251) & \underline{0.0401 (0.0055)} & 0.0405 (0.0052) & 0.0404 (0.005) \\
\dataset{MNIST} & \textbf{0.0316 (0.0017)} & 0.16 (0.0352) & 0.0323 (0.0017) & 0.0324 (0.0017) & \underline{0.0317 (0.0014)} \\
\dataset{Fashion} & \textbf{0.1175 (0.0018)} & 0.2122 (0.0299) & 0.1192 (0.0022) & 0.1197 (0.0022) & \underline{0.1178 (0.0021)} \\
\bottomrule
\end{tabular}

%% file: experiments/RF/bound_table.tex
\begin{tabular}{lcccc}\toprule
Data set & $\FO(\rho_\lambda)$ & $\TND(\rho_{\TND})$ & $\CCTND(\rho_{\CCTND})$ & $\CCPBB(\rho_{\CCPBB})$ \\
\midrule
\dataset{SVMGuide1} & \textbf{0.1079 (0.0079)} & \underline{0.1836 (0.0062)} & 0.1853 (0.0059) & 0.2806 (0.0071) \\
\dataset{Phishing} & \textbf{0.1189 (0.0035)} & \underline{0.1642 (0.0043)} & 0.1674 (0.0042) & 0.2336 (0.005) \\
\dataset{Mushroom} & \textbf{0.0068 (0.0001)} & \underline{0.0353 (0.0002)} & 0.0388 (0.0002) & 0.1121 (0.0006) \\
\dataset{Splice} & \textbf{0.3245 (0.0218)} & \underline{0.4077 (0.0062)} & 0.4247 (0.0065) & 0.6562 (0.0056) \\
\dataset{w1a} & \textbf{0.0424 (0.0015)} & \underline{0.0633 (0.0009)} & 0.0642 (0.0009) & 0.0805 (0.0011) \\
\dataset{Cod-RNA} & \textbf{0.1629 (0.0018)} & \underline{0.1663 (0.0014)} & 0.1698 (0.0014) & 0.19 (0.0018) \\
\dataset{Adult} & \textbf{0.4388 (0.0042)} & 0.5701 (0.0051) & \underline{0.5508 (0.004)} & 0.5976 (0.0042) \\
\dataset{Connect-4} & \textbf{0.5978 (0.0067)} & 0.6831 (0.0039) & \underline{0.6758 (0.0036)} & 0.7112 (0.0038) \\
\dataset{Shuttle} & \textbf{0.0026 (0.0002)} & \underline{0.0078 (0.0002)} & 0.0083 (0.0002) & 0.018 (0.0003) \\
\dataset{Pendigits} & \textbf{0.142 (0.0035)} & \underline{0.1445 (0.0026)} & 0.1504 (0.0042) & 0.2155 (0.003) \\
\dataset{Letter} & \textbf{0.3858 (0.0067)} & \underline{0.4504 (0.0032)} & 0.4513 (0.003) & 0.5134 (0.0039) \\
\dataset{SatImage} & \textbf{0.3762 (0.0075)} & 0.4902 (0.0079) & \underline{0.4851 (0.007)} & 0.6158 (0.0083) \\
\dataset{Sensorless} & 0.0348 (0.0031) & \underline{\textbf{0.0257 (0.0006)}} & 0.0265 (0.0006) & 0.0376 (0.0007) \\
\dataset{USPS} & \textbf{0.3394 (0.0065)} & \underline{0.4059 (0.0048)} & 0.4097 (0.0044) & 0.5086 (0.0042) \\
\dataset{MNIST} & 0.3795 (0.0031) & \underline{\textbf{0.3537 (0.0014)}} & 0.3598 (0.0014) & 0.3853 (0.0014) \\
\dataset{Fashion} & \textbf{0.4806 (0.003)} & 0.5436 (0.0023) & \underline{0.5408 (0.0021)} & 0.5728 (0.0021) \\
\bottomrule
\end{tabular}

%% file: experiments/RF/values_table.tex
\begin{tabular}{lcccccccccc}\toprule
 & \multicolumn{2}{|c|}{$\FO$} & \multicolumn{2}{|c|}{$\TND$} & \multicolumn{3}{|c|}{$\CCTND$} & \multicolumn{3}{|c|}{$\CCPBB$} \\
Data set & \multicolumn{1}{|c}{$\E_\rho[L]$} & \multicolumn{1}{c}{$\E_{\rho^2}[L]$} & \multicolumn{1}{|c}{$\E_\rho[L]$} & \multicolumn{1}{c}{$\E_{\rho^2}[L]$} & \multicolumn{1}{|c}{$\E_\rho[L]$} & \multicolumn{1}{c}{$\E_{\rho^2}[L]$} & \multicolumn{1}{c}{$\mu$} & \multicolumn{1}{|c}{$\E_\rho[L]$} & \multicolumn{1}{c}{$\E_{\rho^2}[L]$} & \multicolumn{1}{c|}{$\mu$} \\
\midrule
\dataset{SVMGuide1} & 0.0325 & 0.0217 & 0.0406 & 0.0185 & 0.0403 & 0.0184 & -0.0527 & 0.0413 & 0.0194 & -0.0258 \\
\dataset{Phishing} & 0.041 & 0.0255 & 0.0486 & 0.0197 & 0.0484 & 0.0196 & -0.0295 & 0.049 & 0.0202 & -0.0125 \\
\dataset{Mushroom} & 0.0 & 0.0 & 0.0002 & 0.0 & 0.0002 & 0.0 & -0.0317 & 0.0002 & 0.0 & -0.01 \\
\dataset{Splice} & 0.1068 & 0.0903 & 0.1564 & 0.0424 & 0.1522 & 0.0415 & -0.057 & 0.16 & 0.044 & 0.0045 \\
\dataset{w1a} & 0.0156 & 0.0123 & 0.0179 & 0.0091 & 0.0179 & 0.009 & -0.0111 & 0.018 & 0.0092 & -0.0065 \\
\dataset{Cod-RNA} & 0.0712 & 0.0602 & 0.0802 & 0.0314 & 0.0803 & 0.0314 & -0.0178 & 0.0815 & 0.0318 & 0.0102 \\
\dataset{Adult} & 0.1995 & 0.1474 & 0.2061 & 0.1184 & 0.2056 & 0.1182 & -0.1216 & 0.2068 & 0.1194 & -0.0918 \\
\dataset{Connect-4} & 0.2824 & 0.2564 & 0.2953 & 0.1523 & 0.2943 & 0.1521 & -0.0959 & 0.2974 & 0.1535 & -0.0615 \\
\dataset{Shuttle} & 0.0003 & 0.0001 & 0.0006 & 0.0002 & 0.0006 & 0.0002 & -0.0044 & 0.0006 & 0.0002 & 0.0 \\
\dataset{Pendigits} & 0.0502 & 0.0346 & 0.061 & 0.0163 & 0.0609 & 0.0163 & -0.0099 & 0.0614 & 0.0166 & 0.0092 \\
\dataset{Letter} & 0.1685 & 0.1249 & 0.1803 & 0.0851 & 0.1797 & 0.0849 & -0.0501 & 0.1816 & 0.0861 & -0.0228 \\
\dataset{SatImage} & 0.1478 & 0.0968 & 0.1612 & 0.0746 & 0.1602 & 0.0741 & -0.1104 & 0.1617 & 0.0755 & -0.0535 \\
\dataset{Sensorless} & 0.0125 & 0.0113 & 0.0192 & 0.0027 & 0.0192 & 0.0027 & 0.0008 & 0.0195 & 0.0027 & 0.01 \\
\dataset{USPS} & 0.1363 & 0.0989 & 0.1517 & 0.0644 & 0.1509 & 0.0641 & -0.053 & 0.1522 & 0.065 & -0.0173 \\
\dataset{MNIST} & 0.1763 & 0.1286 & 0.1837 & 0.075 & 0.1835 & 0.075 & 0.0281 & 0.185 & 0.0756 & 0.037 \\
\dataset{Fashion} & 0.2256 & 0.1715 & 0.2325 & 0.1196 & 0.2322 & 0.1195 & -0.0577 & 0.2334 & 0.1203 & -0.0382 \\
\bottomrule
\end{tabular}

%% file: experiments/mce/risk_table.tex
\begin{tabular}{lcccccc}\toprule
Data set & $L(\MV_{u})$ & $L(h_{best})$ & $L(\MV_{\rho_\lambda})$ & $L(\MV_{\rho_{\TND}})$ & $L(\MV_{\rho_{\CCTND}})$ & $L(\MV_{\rho_{\CCPBB}})$ \\
\midrule
\dataset{SVMGuide1} & 0.0357 (0.005) & 0.0404 (0.0047) & 0.0404 (0.0047) & 0.0352 (0.0051) & 0.0348 (0.0053) & \underline{\textbf{0.0343 (0.0059)}} \\
\dataset{Phishing} & 0.0353 (0.0035) & 0.0459 (0.0058) & 0.0459 (0.0058) & \underline{\textbf{0.0333 (0.0031)}} & 0.0337 (0.0028) & 0.0335 (0.0032) \\
\dataset{Mushroom} & 0.0001 (0.0002) & 0.0002 (0.0004) & \underline{\textbf{0.0 (0.0)}} & 0.0001 (0.0002) & 0.0001 (0.0002) & 0.0001 (0.0002) \\
\dataset{Splice} & 0.1055 (0.0104) & 0.0768 (0.0098) & 0.0768 (0.0098) & 0.075 (0.0093) & 0.0768 (0.0098) & \underline{\textbf{0.069 (0.0082)}} \\
\dataset{w1a} & \textbf{0.0125 (0.0007)} & 0.0153 (0.0009) & 0.0153 (0.0009) & \underline{0.0128 (0.0008)} & 0.0129 (0.0008) & \underline{0.0128 (0.0007)} \\
\dataset{Cod-RNA} & 0.0707 (0.0022) & 0.064 (0.0022) & 0.064 (0.0022) & 0.0552 (0.002) & \underline{\textbf{0.0551 (0.0019)}} & 0.0581 (0.0023) \\
\dataset{Adult} & 0.1627 (0.0036) & 0.1543 (0.0039) & 0.1543 (0.0039) & 0.1563 (0.0042) & \underline{\textbf{0.1541 (0.0039)}} & 0.1566 (0.0048) \\
\dataset{Protein} & 0.3491 (0.0066) & 0.3251 (0.0061) & 0.3251 (0.0061) & \underline{\textbf{0.3176 (0.0052)}} & 0.3251 (0.0061) & 0.3185 (0.0048) \\
\dataset{Connect-4} & 0.2039 (0.0035) & 0.2433 (0.0032) & 0.2433 (0.0032) & \underline{\textbf{0.1989 (0.003)}} & 0.1992 (0.0032) & 0.2018 (0.0037) \\
\dataset{Shuttle} & 0.0012 (0.0002) & \textbf{0.0005 (0.0002)} & \underline{\textbf{0.0005 (0.0002)}} & 0.0006 (0.0002) & 0.0006 (0.0002) & 0.0006 (0.0002) \\
\dataset{Pendigits} & 0.0111 (0.0016) & 0.0092 (0.0017) & 0.0092 (0.0017) & 0.0086 (0.0016) & 0.0087 (0.0016) & \underline{\textbf{0.0085 (0.0019)}} \\
\dataset{Letter} & 0.069 (0.0041) & 0.0673 (0.0052) & 0.0673 (0.0052) & 0.0538 (0.0043) & 0.054 (0.0043) & \underline{\textbf{0.0526 (0.0041)}} \\
\dataset{SatImage} & 0.0997 (0.0069) & 0.1054 (0.0046) & 0.1053 (0.0046) & 0.0939 (0.0061) & 0.0954 (0.0063) & \underline{\textbf{0.093 (0.0059)}} \\
\dataset{Sensorless} & 0.1816 (0.0121) & \textbf{0.0213 (0.0018)} & \underline{\textbf{0.0213 (0.0018)}} & \underline{\textbf{0.0213 (0.0018)}} & 0.1089 (0.2764) & \underline{\textbf{0.0213 (0.0018)}} \\
\dataset{USPS} & 0.0359 (0.0054) & 0.0375 (0.0038) & 0.0375 (0.0038) & \underline{\textbf{0.0324 (0.0044)}} & 0.0326 (0.0042) & 0.0326 (0.0036) \\
\dataset{MNIST} & 0.0356 (0.002) & 0.0349 (0.0017) & 0.0349 (0.0017) & \underline{\textbf{0.0304 (0.0016)}} & \underline{\textbf{0.0304 (0.0017)}} & \underline{\textbf{0.0304 (0.0016)}} \\
\dataset{Fashion} & 0.1341 (0.0019) & 0.154 (0.0028) & 0.154 (0.0028) & \underline{\textbf{0.1323 (0.003)}} & 0.1341 (0.003) & 0.1346 (0.0034) \\
\bottomrule
\end{tabular}

%% file: experiments/mce/bound_table.tex
\begin{tabular}{lcccc}\toprule
Data set & $\FO(\rho_\lambda)$ & $\TND(\rho_{\TND})$ & $\CCTND(\rho_{\CCTND})$ & $\CCPBB(\rho_{\CCPBB})$ \\
\midrule
\dataset{SVMGuide1} & \textbf{0.1133 (0.0053)} & 0.221 (0.0127) & \underline{0.2183 (0.0112)} & 0.3142 (0.0116) \\
\dataset{Phishing} & \textbf{0.1242 (0.0056)} & \underline{0.1957 (0.0075)} & 0.1977 (0.0072) & 0.2658 (0.0074) \\
\dataset{Mushroom} & \textbf{0.0078 (0.0008)} & \underline{0.0412 (0.0019)} & 0.0441 (0.0019) & 0.1162 (0.0026) \\
\dataset{Splice} & \textbf{0.2361 (0.0186)} & 0.4772 (0.0286) & \underline{0.4613 (0.0242)} & 0.6769 (0.0288) \\
\dataset{w1a} & \textbf{0.0392 (0.0015)} & \underline{0.0694 (0.0021)} & 0.0703 (0.0021) & 0.0879 (0.0022) \\
\dataset{Cod-RNA} & \textbf{0.1448 (0.0026)} & 0.2164 (0.0032) & \underline{0.2148 (0.0031)} & 0.2445 (0.003) \\
\dataset{Adult} & \textbf{0.3343 (0.0071)} & 0.5648 (0.0077) & \underline{0.5366 (0.0066)} & 0.5857 (0.0064) \\
\dataset{Protein} & \textbf{0.6944 (0.0057)} & 1.0 (0.0) & \underline{0.9078 (0.0034)} & 1.0 (0.0) \\
\dataset{Connect-4} & \textbf{0.5157 (0.0047)} & 0.7272 (0.0099) & \underline{0.6733 (0.0068)} & 0.7107 (0.0064) \\
\dataset{Shuttle} & \textbf{0.0033 (0.0008)} & \underline{0.0106 (0.0012)} & 0.0111 (0.0012) & 0.0215 (0.0011) \\
\dataset{Pendigits} & \textbf{0.0335 (0.0033)} & \underline{0.0838 (0.0062)} & 0.0856 (0.0061) & 0.1412 (0.0067) \\
\dataset{Letter} & \textbf{0.1591 (0.0053)} & 0.2682 (0.0092) & \underline{0.2627 (0.0084)} & 0.3154 (0.0099) \\
\dataset{SatImage} & \textbf{0.271 (0.0146)} & 0.4908 (0.0123) & \underline{0.4593 (0.011)} & 0.5857 (0.0122) \\
\dataset{Sensorless} & \textbf{0.0523 (0.0031)} & \underline{0.1173 (0.0057)} & 0.2054 (0.2793) & 0.1357 (0.0064) \\
\dataset{USPS} & \textbf{0.1069 (0.0053)} & 0.2183 (0.0074) & \underline{0.2142 (0.0066)} & 0.2932 (0.0084) \\
\dataset{MNIST} & \textbf{0.081 (0.0022)} & 0.139 (0.0049) & \underline{0.1383 (0.0046)} & 0.1574 (0.0051) \\
\dataset{Fashion} & \textbf{0.3291 (0.0033)} & 0.4945 (0.0066) & \underline{0.4709 (0.0049)} & 0.5049 (0.0045) \\
\bottomrule
\end{tabular}

%% file: experiments/mce/values_table.tex
\begin{tabular}{lcccccccccc}\toprule
 & \multicolumn{2}{|c|}{$\FO$} & \multicolumn{2}{|c|}{$\TND$} & \multicolumn{3}{|c|}{$\CCTND$} & \multicolumn{3}{|c|}{$\CCPBB$} \\
Data set & \multicolumn{1}{|c}{$\E_\rho[L]$} & \multicolumn{1}{c}{$\E_{\rho^2}[L]$} & \multicolumn{1}{|c}{$\E_\rho[L]$} & \multicolumn{1}{c}{$\E_{\rho^2}[L]$} & \multicolumn{1}{|c}{$\E_\rho[L]$} & \multicolumn{1}{c}{$\E_{\rho^2}[L]$} & \multicolumn{1}{c}{$\mu$} & \multicolumn{1}{|c}{$\E_\rho[L]$} & \multicolumn{1}{c}{$\E_{\rho^2}[L]$} & \multicolumn{1}{c|}{$\mu$} \\
\midrule
\dataset{SVMGuide1} & 0.0365 & 0.031 & 0.0457 & 0.026 & 0.0439 & 0.0258 & -0.0691 & 0.046 & 0.0267 & -0.0362 \\
\dataset{Phishing} & 0.0447 & 0.0383 & 0.059 & 0.0262 & 0.0533 & 0.0262 & -0.0419 & 0.059 & 0.0266 & -0.0158 \\
\dataset{Mushroom} & 0.0001 & 0.0 & 0.0041 & 0.0003 & 0.0029 & 0.0002 & -0.0289 & 0.0046 & 0.0003 & -0.008 \\
\dataset{Splice} & 0.0754 & 0.0754 & 0.1265 & 0.0552 & 0.1062 & 0.0548 & -0.19 & 0.1331 & 0.0578 & -0.0672 \\
\dataset{w1a} & 0.0147 & 0.0134 & 0.0204 & 0.0106 & 0.0195 & 0.0106 & -0.0125 & 0.0192 & 0.0108 & -0.0065 \\
\dataset{Cod-RNA} & 0.0638 & 0.0572 & 0.0737 & 0.0425 & 0.0717 & 0.0425 & -0.0356 & 0.0813 & 0.0442 & -0.02 \\
\dataset{Adult} & 0.1502 & 0.1484 & 0.2049 & 0.1181 & 0.1726 & 0.1224 & -0.1771 & 0.1907 & 0.1202 & -0.1168 \\
\dataset{Protein} & 0.3224 & 0.3153 & 0.4052 & 0.2645 & 0.324 & 0.3017 & -1.2391 & 0.4182 & 0.267 & -0.5 \\
\dataset{Connect-4} & 0.2438 & 0.236 & 0.2612 & 0.1634 & 0.2534 & 0.1638 & -0.221 & 0.2628 & 0.165 & -0.1912 \\
\dataset{Shuttle} & 0.0005 & 0.0004 & 0.0019 & 0.0004 & 0.0014 & 0.0004 & -0.0047 & 0.0021 & 0.0004 & 0.0 \\
\dataset{Pendigits} & 0.008 & 0.0069 & 0.0156 & 0.0062 & 0.0133 & 0.0061 & -0.0299 & 0.0144 & 0.0063 & -0.0118 \\
\dataset{Letter} & 0.0644 & 0.0587 & 0.0768 & 0.0454 & 0.0723 & 0.0454 & -0.0641 & 0.086 & 0.0466 & -0.0407 \\
\dataset{SatImage} & 0.1032 & 0.0928 & 0.1246 & 0.0766 & 0.1188 & 0.0764 & -0.1647 & 0.1284 & 0.0783 & -0.0995 \\
\dataset{Sensorless} & 0.0208 & 0.0208 & 0.0345 & 0.0198 & 0.1161 & 0.1036 & -0.027 & 0.0445 & 0.0202 & -0.0122 \\
\dataset{USPS} & 0.036 & 0.0326 & 0.0478 & 0.028 & 0.0434 & 0.0278 & -0.0676 & 0.0496 & 0.0288 & -0.0357 \\
\dataset{MNIST} & 0.0345 & 0.0304 & 0.0428 & 0.026 & 0.0403 & 0.026 & -0.0272 & 0.0482 & 0.0265 & -0.017 \\
\dataset{Fashion} & 0.1528 & 0.1488 & 0.1665 & 0.1081 & 0.1636 & 0.1084 & -0.1173 & 0.1724 & 0.1092 & -0.094 \\
\bottomrule
\end{tabular}